\documentclass[a4paper]{article}
\usepackage{hyperref}
\usepackage[english]{babel}
\usepackage[T1]{fontenc}
\usepackage{amsfonts}

\usepackage[a4paper,top=3cm,bottom=2cm,left=3cm,right=3cm,marginparwidth=1.75cm]{geometry}

\usepackage{algorithm}	
\usepackage[noend]{algpseudocode}
\usepackage[square,sort,comma,numbers]{natbib}
\usepackage{booktabs}
\usepackage{amssymb}
\usepackage{amsmath}
\usepackage{amsthm}
\usepackage{graphicx, color}
\usepackage{tikz}
\usetikzlibrary{shapes,arrows}
\usepackage{verbatim}
\usepackage{subcaption}
\usepackage[flushleft]{threeparttable}
 \newtheorem{Theorem}{Theorem}

 \newtheorem{Proposition}{Proposition}

  \DeclareMathOperator*{\argmax}{\arg\!\max}
 
\begin{document}
\title{Context-Aware Generative Adversarial Privacy}
\author{
	Chong Huang\footnote{C. Huang and L. Sankar are with the School of Electrical, Computer, and Energy Engineering at Arizona State University, Tempe, AZ} \footnotemark[3], Peter Kairouz\footnote{P. Kairouz, X. Chen, and R. Rajagopal are with the Department of Electrical Engineering at Stanford University, Stanford, CA} \footnote{Equal contributions}, Xiao Chen\footnotemark[2], Lalitha Sankar\footnotemark[1], and Ram Rajagopal\footnotemark[2]
}
\date{}

\maketitle
\pagenumbering{arabic}
\begin{abstract}
Preserving the utility of published datasets while simultaneously providing provable privacy guarantees is a well-known challenge. On the one hand, context-free privacy solutions, such as differential privacy, provide strong privacy guarantees, but often lead to a significant reduction in utility. On the other hand, context-aware privacy solutions, such as information theoretic privacy, achieve an improved privacy-utility tradeoff, but assume that the data holder has access to dataset statistics. We circumvent these limitations by introducing a novel context-aware privacy framework called generative adversarial privacy (GAP).
GAP leverages recent advancements in generative adversarial networks (GANs) to allow the data holder to learn privatization schemes from the dataset itself. Under GAP, learning the privacy mechanism is formulated as a constrained minimax game between two players: a privatizer that sanitizes the dataset in a way that limits the risk of inference attacks on the individuals' private variables, and an adversary that tries to infer the private variables from the sanitized dataset. To evaluate GAP's performance, we investigate two simple (yet canonical) statistical dataset models: (a) the binary data model, and (b) the binary Gaussian mixture model. For both models, we derive game-theoretically optimal minimax privacy mechanisms, and show that the privacy mechanisms learned from data (in a generative adversarial fashion) match the theoretically optimal ones. This demonstrates that our framework can be easily applied in practice, even in the absence of dataset statistics.

\textbf{Keywords-} {Generative Adversarial Privacy; Generative Adversarial Networks; Privatizer Network; Adversarial Network; Statistical Data Privacy; Differential Privacy; Information Theoretic Privacy; Mutual Information Privacy; Error Probability Games; Machine Learning}
\end{abstract}

\section{Introduction}

The explosion of information collection across a variety of electronic platforms is enabling the use of \textit{inferential machine learning} (ML) and artificial intelligence to guide consumers through a myriad of choices and decisions in their daily lives. In this era of artificial intelligence, data is quickly becoming the most valuable resource \citep{economist2017}. Indeed, large scale datasets provide tremendous \textit{utility} in helping researchers design state-of-the-art machine learning algorithms that can learn from and make predictions on real life data. Scholars and researchers are increasingly demanding access to larger datasets that allow them to learn more sophisticated models. Unfortunately, more often than not, in addition to containing \textit{public} information that can be published, large scale datasets also contain \textit{private} information about participating individuals (see Figure \ref{fig:database}). Thus, data collection and curation organizations are reluctant to release such datasets before carefully \textit{sanitizing} them, especially in light of recent public policies on data sharing \citep{NPRS2016,EUGDPR2017}.

\begin{figure}
	\centering
	\includegraphics[width=0.8\textwidth]{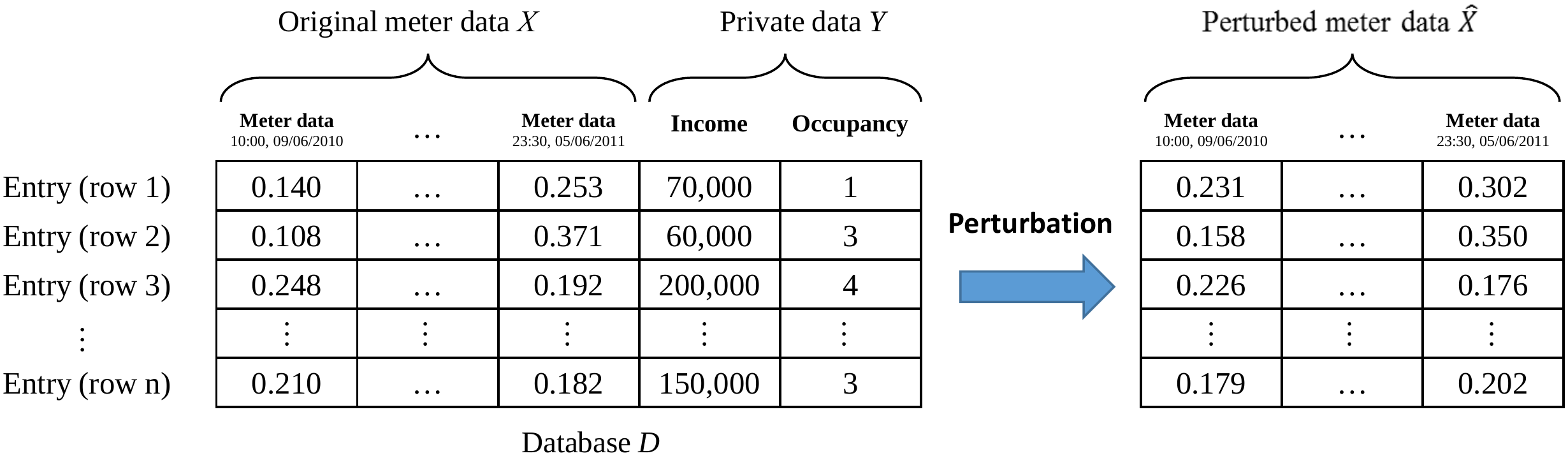}
	\caption{An example privacy preserving mechanism for smart meter data}
	\label{fig:database}
\end{figure}

To protect the privacy of individuals, datasets are typically anonymized before their release. This is done by stripping off personally identifiable information (e.g., first and last name, social security number, IDs, etc.) \citep{samarati1998protecting,sweeney2002k,li2007t}. Anonymization, however, does not provide immunity against correlation and linkage attacks \citep{narayanan2008robust, Harmanci2016}. Indeed, several successful attempts to re-identify individuals from anonymized datasets have been reported in the past ten years. For instance, \citep{narayanan2008robust} were able to successfully de-anonymize watch histories in the Netflix Prize, a public recommender system competition. In a more recent attack, \citep{Sweeney2013} showed that participants of an anonymized DNA study were identified by linking their DNA data with the publicly available Personal Genome Project dataset. Even more recently, \citep{Finn2015} successfully designed re-identification attacks on anonymized fMRI imaging datasets. Other annoymization techniques, such as generalization \cite{lefevre2005incognito,bayardo2005data,fung2007anonymizing} and suppression \cite{iyengar2002transforming,samarati2001protecting,wang2007handicapping}, also cannot prevent an adversary from performing the sensitive linkages or recover private information from published datasets \cite{fung2010privacy}.

Addressing the shortcomings of anonymization techniques requires data randomization. In recent years, two randomization-based approaches with provable \textit{statistical privacy} guarantees have emerged: (a) context-free approaches that assume worst-case dataset statistics and adversaries; (b) context-aware approaches that explicitly model the dataset statistics and adversary's capabilities.

{\bf Context-free privacy.}
One of the most popular context-free notions of privacy is \textit{differential privacy} (DP) \citep{Dwork_DP, Dwork_DP_Survey,Dwork2014}. DP, quantified by a leakage parameter $\epsilon$\footnote{Smaller $\epsilon \in[0,\infty)$ implies smaller leakage and stronger privacy guarantees.}, restricts distinguishability between \textit{any} two ``neighboring'' datasets from the published data. DP provides strong, context-free theoretical guarantees against worst-case adversaries. However, training machine learning models on randomized data with DP guarantees often leads to a significantly reduced utility
and comes with a tremendous hit in sample complexity \citep{Fienberg2010,WangLeeKifer2015,UhleropSlavkovicFienberg13,YuFienbergSlavkovicUhler2014,KarwaSlavkovic2016,Duchi2013a,duchi2013local,DuchiWainwrightJordan2016,kairouz2016JMLR,Kairouz2016a,YeBarg2017,Raval2017,Hayes2017} in the desired leakage regimes. For example, learning population level histograms under local DP suffers from a stupendous increase in sample complexity by a factor proportional to the size of the dictionary \citep{Kairouz2016a,duchi2013local,kairouz2016JMLR}.

{\bf Context-aware privacy.}
Context-aware privacy notions have been so far studied by information theorists
under the rubric of \textit{information theoretic} (IT) privacy
 \citep{Yamamoto1983,Rebollo-Monedero2010,Varodayan2011,Sankar2011,PinCalmon2012,Sankar_TIFS_2013,CalmonVMCDT13,SankarSmartMeter2013,Salamatian2015,Liao2016,calmon2014allerton,Asoodeh2014,Calmon2015a,Basciftci2016,Kalantari2016a,Kalantari2017,Kalantari2017a,Asoodeh2015,Asoodeh2016,Asoodeh2016a,Moraffah2017}. IT privacy has predominantly been quantified by mutual information (MI) which models how well an adversary, with access to the released data, can refine its belief about the private features of the data. Recently, Issa \textit{et al.} introduced \textit{maximal leakage} (MaxL) to quantify leakage to a strong adversary capable of guessing any function of the dataset \citep{IssaKW16}. They also showed that their adversarial model can be generalized to encompass local DP (wherein the mechanism ensures limited distinction for \textit{any} pair of entries\textemdash a stronger DP notion without a neighborhood constraint \citep{warner1965randomized,duchi2013local})~\citep{Issa2017}. When one restricts the adversary to guessing specific private features (and not all functions of these features), the resulting adversary is a maximum \textit{a posteriori} (MAP) adversary that has been studied by Asoodeh \textit{et al.} in \citep{Asoodeh2016,Asoodeh2016a,Asoodeh2017arxiv,Asoodeh2017}.  Context-aware data perturbation techniques have also been studied in privacy preserving cloud computing \citep{kerschbaum2015frequency,dong2016secure,dong2014prada}.

Compared to context-free privacy notions, context-aware privacy notions achieve a better privacy-utility tradeoff by incorporating the statistics of the dataset and placing reasonable restrictions on the capabilities of the adversary. However, using information theoretic quantities (such as MI) as privacy metrics requires learning the parameters of the privatization mechanism in a data-driven fashion that involves minimizing an empirical information theoretic loss function. This task is remarkably challenging in practice \citep{Alemi2017,Giraldo2013arxiv,ZhangOzay2017arxiv,Theis2017,Moon2017}.

{\bf Generative adversarial privacy.}
Given the challenges of existing privacy approaches, we take a fundamentally new approach towards enabling private data publishing with guarantees on both privacy and utility. Instead of adopting worst-case, context-free notions of data privacy (such as differential privacy), we introduce a novel context-aware model of privacy that allows the designer to cleverly add noise where it matters. An inherent challenge in taking a context-aware privacy approach is that it requires having access to priors, such as joint distributions of public and private variables. Such information is hardly ever present in practice. To overcome this issue, we take a \textit{data-driven approach} to context-aware privacy. We leverage recent advancements in generative adversarial networks (GANs) to introduce a unified framework for context-aware privacy called \textit{generative adversarial privacy} (GAP). Under GAP, the parameters of a generative model, representing the privatization mechanism, are learned from the data itself.




\subsection{Our Contributions}
We investigate a setting where a data holder would like to publish a dataset $\mathcal{D}$ in a privacy preserving fashion. Each row in $\mathcal{D}$ contains both private variables (represented by $Y$) and public variables (represented by $X$). The goal of the data holder is to generate $\hat{X}$ in a way such that: (a) $\hat{X}$ is as good of a representation of $X$ as possible, and (b) an adversary cannot use $\hat{X}$ to reliably infer $Y$. To this end, we present GAP, a unified framework for context-aware privacy that includes existing information-theoretic privacy notions. Our formulation is inspired by GANs \cite{mirza2014conditional,schmidhuber1992learning,goodfellow2014generative} and error probability games \cite{morris1980single, shamai1992worst, morris1981single, morris1990random, root1961communications}. It includes two learning blocks: a \textit{privatizer}, whose task is to output a sanitized version of the public variables (subject to some distortion constraints); and an \textit{adversary}, whose task is to learn the private variables from the sanitized data. The privatizer and adversary achieve their goals by competing in a constrained minimax, zero-sum game. On the one hand, the privatizer (a conditional generative model) is designed to minimize the adversary's performance in inferring $Y$ reliably. On the other hand, the adversary (a classifier) seeks to find the best inference strategy that maximizes its performance. This generative adversarial framework is represented in Figure \ref{fig:GAP}.

\tikzstyle{block} = [draw, fill=white!20, rectangle,
    minimum height=3em, minimum width=6em]
\tikzstyle{sum} = [draw, fill=white!20, circle, node distance=1cm]
\tikzstyle{input} = [coordinate]
\tikzstyle{output} = [coordinate]
\tikzstyle{pinstyle} = [pin edge={to-,thin,black}]

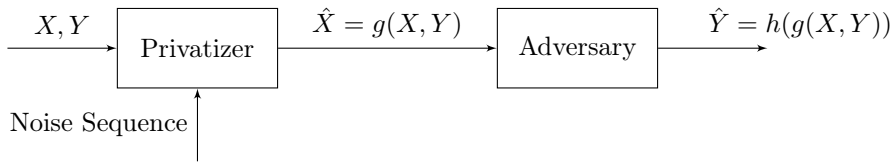
\begin{figure}
\begin{center}
\begin{tikzpicture}[auto, node distance=2.5cm,>=latex']
	\node [input, name=input] {};
	\node [block, right of=input] (generator) {Privatizer};
	\node [output, right of=generator] (output_generator) {};
	\node[block, right of=output_generator] (discriminator) {Adversary};
	\node [input, name=image, below of=generator, node distance = 1.5cm] {};
	\node [output, right of=discriminator] (output_discriminator) {};
	 \draw [draw,->] (input) -- node {$X, Y$} (generator);
         \draw [draw,->] (generator) -- node {$\hat{X} =g(X,Y)$} (discriminator);
        \draw [->] (discriminator) -- node[pos=1.3] {$\hat{Y} = h(g(X,Y))$} (output_discriminator);
        \draw [->] (image) -- node {Noise Sequence} (generator);
\end{tikzpicture}
\end{center}
\caption{Generative Adversarial Privacy}\label{fig:GAP}
\end{figure}

At the core of GAP is a {loss} function\footnote{We quantify the adversary's performance via a {loss} function and the quality of the released data via a distortion function.} that captures how well an adversary does in terms of inferring the private variables. Different {loss} functions lead to different adversarial models. We focus our attention on two types of {loss functions: (a) a $0$-$1$ loss that leads to a \textit{maximum a posteriori probability} (MAP) adversary, and (b) an empirical \textit{log-loss} that leads to a \textit{minimum cross-entropy} adversary}. Ultimately, our goal is to show that our data-driven approach can provide privacy guarantees against a MAP adversary. However, derivatives of a $0$-$1$ {loss} function are ill-defined. To overcome this issue, the ML community uses the more analytically tractable {log-loss} function. We do the same by choosing the {log-loss} function as the adversary's {loss} function in the data-driven framework. We show that it leads to a performance that matches the performance of game-theoretically optimal mechanisms under a MAP adversary. We also show that GAP recovers mutual information privacy when a {$\log$-loss} function is used (see Section \ref{sec:lossfunctions}).

To showcase the power of our context-aware, data-driven framework, we investigate two simple, albeit canonical, statistical dataset models: (a) the binary data model,  and (b) the binary Gaussian mixture model. Under the binary data model, both $X$ and $Y$ are binary. Under the binary Gaussian mixture model, $Y$ is binary whereas $X$ is conditionally Gaussian. For both models, we derive and compare the performance of game-theoretically optimal privatization mechanisms with those that are directly learned from data (in a generative adversarial fashion).

For the above-mentioned statistical dataset models, we present two approaches towards designing privacy mechanisms: (i) private-data dependent (PDD) mechanisms, where the privatizer uses both the public and private variables, and (ii) private-data independent (PDI) mechanisms, where the privatizer only uses the public variables. We show that the PDD mechanisms lead to a superior privacy-utility tradeoff.

\subsection{Related Work}
In practice, a context-free notion of privacy (such as DP) is desirable because it places no restrictions on the dataset statistics or adversary's strength. This explains why DP has been remarkably successful in the past ten years, and has been deployed in array of systems, including Google's Chrome browser \citep{erlingsson2014rappor} and Apple's iOS \citep{apple1}. Nevertheless, because of its strong context-free nature, DP has suffered from a sequence of impossibility results. These results have made the deployment of DP with a reasonable leakage parameter practically impossible. Indeed, it was recently reported that Apple's DP implementation suffers from several limitations \textemdash most notable of which is Apple's use of unacceptably large leakage parameters \citep{tang2017privacy}.

Context-aware privacy notions can exploit the structure and statistics of the dataset to design mechanisms matched to both the data and adversarial models. In this context, information-theoretic metrics for privacy are naturally well suited. In fact, the adversarial model determines the appropriate information metric: an estimating adversary that minimizes mean square error is captured by $\chi^2$-squared measures \citep{CalmonVMCDT13}, a belief refining adversary is captured by MI \citep{Sankar_TIFS_2013}, an adversary that can make a hard MAP decision for a specific set of private features is captured by the Arimoto MI of order $\infty$ \citep{Asoodeh2017arxiv,Asoodeh2017}, and an adversary that can guess any function of the private features is captured by the maximal (over all distributions of the dataset for a fixed support) Sibson information of order $\infty$ \citep{IssaKW16,Issa2017}.

Information-theoretic metrics, and in particular MI privacy, allow the use of Fano's inequality and its variants \citep{alphaMI_verdu} to bound the rate of learning the private variables for a variety of learning metrics, such as error probability and minimum mean-squared error (MMSE). Despite the strength of MI in providing statistical utility as well as capturing a fairly strong adversary that involves refining beliefs, in the absence of priors on the dataset, using MI as an empirical loss function leads to computationally intractable procedures when learning the optimal parameters of the privatization mechanism from data. Indeed, training algorithms with empirical information theoretic loss functions is a challenging problem that has been explored in specific learning contexts, such as determining randomized encoders for the information bottleneck problem  \citep{Alemi2017} and designing deep auto-encoders using a rate-distortion paradigm \citep{Giraldo2013arxiv,ZhangOzay2017arxiv,Theis2017}. Even in these specific contexts, variational approaches were taken to minimize/maximize a surrogate function instead of minimizing/maximizing an empirical mutual information loss function directly \citep{Sugiyama2013}. In an effort to bridge theory and practice, we present a general data-driven framework to design privacy mechanisms that can capture a range of information-theoretic privacy metrics as loss functions. We will show how our framework leads to very practical (generative adversarial) data-driven formulations that match their corresponding theoretical formulations.


In the context of publishing datasets with privacy and utility guarantees, a number of similar approaches have been recently considered. We briefly review them and clarify how our work is different. In \citep{Xu2017}, the authors consider linear privatizer and adversary models by adding noise in directions that are orthogonal to the public features in the hope that the ``spaces'' of the public and private features are orthogonal (or nearly orthogonal). This allows the privatizer to achieve full privacy without sacrificing utility. However, this work is restrictive in the sense that it requires the public and private features to be nearly orthogonal. Furthermore, this work provides no rigorous quantification of privacy and only investigates a limited class of linear adversaries and privatizers.

DP-based obfuscators for data publishing have been considered in \cite{hamm2016minimax, LiuDEEPProtectArxiv2017}. The author in \cite{hamm2016minimax} considers a deterministic, compressive mapping of the input data with differentially private noise added either before or after the mapping. The mapping rule is determined by a data-driven methodology to design minimax filters that allow non-malicious entities to learn some public features from the filtered data, while preventing malicious entities from learning other private features. The approach in \cite{LiuDEEPProtectArxiv2017} relies on using deep auto-encoders to determine the relevant feature space to add differentially private noise to, eliminating the need to add noise to the original data. After noise adding, the original signal is reconstructed. These novel approaches leverage minimax filters and deep auto-encoders to incorporate a notion of context-aware privacy and achieve better privacy-utility tradeoffs while using DP to enforce privacy. However, DP will still incur an insurmountable utility cost since it assumes worst-case dataset statistics. Our approach captures a broader class of randomization-based mechanisms via a generative model which allows the privatizer to tailor the noise to the statistics of the dataset.

Our work is also closely related to adversarial neural cryptography \cite{abadi2016learning}, learning censored representations \cite{edwards2015censoring}, and privacy preserving image sharing \citep{Raval2017}, in which adversarial learning is used to learn how to protect communications by encryption or hide/remove sensitive information. Similar to these problems, our model includes a minimax formulation and uses adversarial neural networks to learn privatization schemes. However, in \cite{edwards2015censoring, Raval2017}, the authors use non-generative auto-encoders to remove sensitive information, which do not have an obvious generative interpretation. Instead, we use a GANs-like approach to learn privatization schemes that prevent an adversary from inferring the private data. Moreover, these papers consider a Lagrangian formulation for the utility-privacy tradeoff that the obfuscator computes. We go beyond these works by studying a game-theoretic setting with constrained optimization, which provides a specific privacy guarantee for a fixed distortion. We also compare the performance of the privatization schemes learned in an adversarial fashion with the game-theoretically optimal ones.

We use conditional generative models to represent privatization schemes. Generative models have recently received a lot of attention in the machine learning community \citep{smolensky1986information, hinton2009deep, schmidhuber1992learning, goodfellow2014generative,mirza2014conditional}. Ultimately, deep generative models hold the promise of discovering and efficiently internalizing the statistics of the target signal to be generated. State-of-the-art generative models are trained in an adversarial fashion \citep{goodfellow2014generative,mirza2014conditional}: the generated signal is fed into a discriminator which attempts to distinguish whether the data is real (i.e., sampled from the true underlying distribution) or synthetic (i.e., generated from a low dimensional noise sequence). Training generative models in an adversarial fashion has proven to be successful in computer vision and enabled several exciting applications. Analogous to how the generator is trained in GANs, we train the privatizer in an adversarial fashion by making it compete with an attacker.

\subsection{Outline}
\label{subsec:outline}

The remainder of our paper is organized as follows. We formally present our GAP model in Section \ref{sec:model}. We also show how, as a special case, it can recover several information theoretic notions of privacy. We then study a simple (but canonical) binary dataset model in Section \ref{sec:binary}. In particular, we present theoretically optimal PDD and PDI privatization schemes, and show how these schemes can be learned from data using a generative adversarial network. In Section \ref{sec:gaussian}, we investigate binary Gaussian mixture dataset models, and provide a variety of privatization schemes. We comment on their theoretical performance and show how their parameters can be learned from data in a generative adversarial fashion. Our proofs are deferred to sections \ref{binaryproof}, \ref{gaussianscheme0proof}, and \ref{gaussianscheme1proof} of the Appendix. We conclude our paper in Section \ref{sec:conclusion} with a few remarks and interesting extensions.

\section{Generative Adversarial Privacy Model}
\label{sec:model}
We consider a dataset $\mathcal{D}$ which contains both public and private variables for $n$ individuals (see Figure~\ref{fig:database}). We represent the public variables by a random variable $X\in\mathcal{X}$, and the private variables (which are typically correlated with the public variables) by a random variable $Y\in\mathcal{Y}$. Each dataset entry contains a pair of public and private variables denoted by $(X,Y)$. Instances of $X$ and $Y$ are denoted by $x$ and $y$, respectively. We assume that each entry pair $(X,Y)$ is distributed according to $P(X,Y)$, and is independent from other entry pairs in the dataset. Since the dataset entries are independent of each other, we restrict our attention to memoryless mechanisms: privacy mechanisms that are applied on each data entry separately. Formally, we define the privacy mechanism as a randomized mapping given by
$$
g(X, Y): \mathcal{X}\times\mathcal{Y}\rightarrow \mathcal{X}.
$$

We consider two different types of privatization schemes: (a) private data dependent (PDD) schemes, and (b) private data independent (PDI) schemes. A privatization mechanism is PDD if its output is dependent on both $Y$ and $X$. It is PDI if its output only depends on $X$. PDD mechanisms are naturally superior to PDI mechanisms. We show, in sections \ref{sec:binary} and \ref{sec:gaussian}, that there is a sizeable gap in performance between these two approaches.

In our proposed GAP framework, the privatizer is pitted against an adversary. We model the interactions between the privatizer and the adversary as a non-cooperative game. For a fixed $g$, the goal of the adversary is to reliably infer $Y$ from $g(X, Y)$ using a strategy $h$. For a fixed adversarial strategy $h$, the goal of the privatizer is to design $g$ in a way that minimizes the adversary's capability of inferring the private variable from the perturbed data. The optimal privacy mechanism is obtained as an equilibrium point at which both the privatizer and the adversary can not improve their strategies by unilaterally deviating from the equilibrium point.

\subsection{Formulation}
Given the output $\hat{X}=g(X,Y)$ of a privacy mechanism $g(X,Y)$, we define $\hat{Y}=h(g(X,Y))$ to be the adversary's inference of the private variable $Y$ from $\hat{X}$. To quantify the effect of adversarial inference,  for a given public-private pair $(x,y)$, we model the {loss} of the adversary as
$$\ell(h(g(X=x,Y=y)),Y=y): \mathcal{Y}\times\mathcal{Y}\rightarrow \mathbb{R}.$$
Therefore, the {expected loss} of the adversary with respect to (\textit{w.r.t.}) $X$ and $Y$ is defined to be
\begin{align}
\label{eq:systemadversaryloss}
L(h, g)\triangleq \mathbb{E}[\ell(h(g(X,Y)),Y)],
\end{align}
where the expectation is taken over $P(X,Y)$ and the randomness in $g$ and $h$.

Intuitively, the privatizer would like to minimize the adversary's ability to learn $Y$ reliably from the published data. This can be trivially done by releasing an $\hat{X}$ independent of $X$. However, such an approach provides no utility for data analysts who want to learn non-private variables from $\hat{X}$. To overcome this issue, we capture the loss incurred by privatizing the original data via a distortion function 
$d(\hat{x},x): \mathcal{X}\times\mathcal{X}\rightarrow \mathbb{R}$, which measures how far the original data $X = x$ is from the privatized data $\hat{X} = \hat{x}$. Thus, the average distortion under $g(X,Y)$ is $\mathbb{E}[d( g(X,Y),X)]$, where the expectation is taken over $P(X,Y)$ and the randomness in $g$.

On the one hand, the data holder would like to find a privacy mechanism $g$ that is both privacy preserving (in the sense that it is difficult for the adversary to learn $Y$ from $\hat{X}$) and utility preserving (in the sense that it does not distort the original data too much). On the other hand, for a fixed choice of privacy mechanism $g$, the adversary would like to find a (potentially randomized) function $h$ that {minimizes its expected loss, which is equivalent to maximizing the negative of the expected loss}. To achieve these two opposing goals, we model the problem as a constrained minimax game between the privatizer and the adversary:
\begin{align}
\label{eq:generalopt}
\min_{g(\cdot)}\max_{h(\cdot)} \quad & -L(h,g) \\ \nonumber
s.t.  \quad & \mathbb{E}[d( g(X,Y),X)]\le D,
\end{align}
where the constant $D\ge0$ determines the allowable distortion for the privatizer and the expectation is taken over $P(X,Y)$ and the randomness in $g$ and $h$.

\subsection{GAP under Various Loss Functions}
\label{sec:lossfunctions}
The above formulation places no restrictions on the adversary. Indeed, different loss functions and decision rules lead to different adversarial models. In what follows, we will discuss a variety of loss functions under hard and soft decision rules, and show how our GAP framework can recover several popular information theoretic privacy notions.

\textbf{Hard Decision Rules. } When the adversary adopts a hard decision rule, $h(g(X, Y))$ is an estimate of $Y$. Under this setting, we can choose $\ell(h(g(X, Y)), Y)$ in a variety of ways. For instance, if $Y$ is continuous, the adversary can attempt to minimize the difference between the estimated and true private variable values. This can be achieved by considering a squared loss function
\begin{align}
\ell(h(g(X, Y)),Y)= (h(g(X, Y)) - Y)^2,
\end{align}
which is known as the $\ell_2$ loss. In this case, one can verify that the adversary's optimal decision rule is $h^* = \mathbb{E}[Y|g(X, Y)]$, which is the conditional mean of $Y$ given $g(X,Y)$. Furthermore, under the adversary's optimal decision rule, the minimax problem in \eqref{eq:generalopt} simplifies to
$$\min_{g(\cdot)} - \text{mmse}(Y|g(X,Y)) = - \max_{g(\cdot)} \text{mmse}(Y|g(X, Y)),$$
subject to the distortion constraint. Here $\text{mmse}(Y|g(X,Y))$ is the resulting minimum mean square error (MMSE) under $h^* = \mathbb{E}[Y|g(X,Y)]$. Thus, under the $\ell_2$ loss, GAP provides privacy guarantees against an MMSE adversary. On the other hand, when $Y$ is discrete (e.g., age, gender, political affiliation, etc), the adversary can attempt to maximize its classification accuracy. This is achieved by considering a $0$-$1$ loss function \citep{nguyen2013algorithms} given by
\begin{align}
\label{eq:0-1loss}
\ell(h(g(X,Y)),Y)=
\left\{
\begin{array}{ll}
0  & \mbox{if } h(g(X,Y))=Y \\
1  & \mbox{otherwise }
\end{array}
\right..
\end{align}
In this case, one can verify that the adversary's optimal decision rule is the \textit{maximum a posteriori probability} (MAP) decision rule: $h^* = \argmax_{y \in \mathcal{Y}} P(y|g(X,Y))$, with ties broken uniformly at random. Moreover, under the MAP decision rule, the minimax problem in \eqref{eq:generalopt} reduces to
\begin{align}
\label{eq:mapadversaryaccuracy}
 \min\limits_{g(\cdot)} - (1 - \max\limits_{y \in \mathcal{Y}} P(y, g(X,Y))) = \min\limits _{g(\cdot)} \max_{y \in \mathcal{Y}} P(y, g(X,Y))  - 1,
\end{align}
subject to the distortion constraint. Thus, under a $0$-$1$ loss function, the GAP formulation provides privacy guarantees against a MAP adversary. 

\textbf{Soft Decision Rules. } Instead of a \textit{hard decision} rule, we can also consider a broader class of \textit{soft decision} rules where $h(g(X,Y))$ is a distribution over $\mathcal{Y}$; i.e., $h(g(X,Y)) = P_h(y|g(X,Y))$ for $y \in \mathcal{Y}$. In this context, we can analyze the performance under a  $\log$-loss
\begin{align}
\label{eq:infologloss}
\ell(h(g(X,Y)),y) = \log \frac{1}{P_h(y|g(X,Y))}.
\end{align}
In this case, the objective of the adversary simplifies to
$$ \max\limits_{h(\cdot)}  -\mathbb{E}[\log \frac{1}{P_h(y|g(X,Y))}] = - H(Y|g(X,Y)),$$
and that the maximization is attained at $P^*_h(y|g(X,Y)) = P(y|g(X,Y))$. Therefore, the optimal adversarial decision rule is determined by the true conditional distribution $P(y|g(X,Y))$, which we assume is known to the data holder in the game-theoretic setting. Thus, under the $\log$-loss function, the minimax optimization problem in \eqref{eq:generalopt} reduces to
$$ \min_{g(\cdot)}  - H(Y|g(X,Y)) = \min_{g(\cdot)} I(g(X,Y);Y) - H(Y),$$
subject to the distortion constraint. Thus, under the $\log$-loss in \eqref{eq:infologloss}, GAP is equivalent to using MI as the privacy metric \citep{PinCalmon2012}.

The $0$-$1$ loss captures a strong guessing adversary; in contrast, log-loss or information-loss models a belief refining adversary. Next, we consider a more general $\alpha$-loss function \cite{LiaoReport2017} that allows continuous interpolation between these extremes via
\begin{align}
\label{poly_loss}
\ell(h(g(X,Y)),y) = \frac{\alpha}{\alpha  - 1} \left( 1- P_h(y|g(X,Y))^{1 - \frac{1}{\alpha}}\right),
\end{align}
for any $\alpha > 1$. As shown in \cite{LiaoReport2017}, for very large $\alpha$ ($\alpha \to \infty$), this loss approaches that of the $0$-$1$ (MAP) adversary. As $\alpha$ decreases, the convexity of the loss function encourages the estimator $\hat{Y}$ to be probabilistic, as it increasingly rewards correct inferences of lesser and lesser likely outcomes (in contrast to a hard decision rule by a MAP adversary of the most likely outcome) conditioned on the revealed data. As $\alpha\to 1$, \eqref{poly_loss} yields the logarithmic loss, and the optimal belief $P_{\hat{Y}}$ is simply the posterior belief.
Denoting $H^{\text{a}}_{\alpha}(Y|g(Y,X))$ as the Arimoto conditional entropy of order $\alpha$, one can verify that \cite{LiaoReport2017}
$$ \max_{h(\cdot)}  -\mathbb{E}\bigg[\frac{\alpha}{\alpha  - 1} \left( 1- P_h(y|g(X,Y))^{1 - \frac{1}{\alpha}}\right)\bigg] = - H^{\text{a}}_{\alpha}(Y|g(X,Y)),$$ which is achieved by a `$\alpha$-tilted' conditional distribution \cite{LiaoReport2017}
$$P^*_h(y|g(X,Y)) = \frac{P(y|g(X,Y))^\alpha}{\sum\limits_{y\in\mathcal{Y}} P(y|g(X,Y))^\alpha}.$$
Under this choice of a decision rule, the objective of the minimax optimization in \eqref{eq:generalopt} reduces to
\begin{equation}
\label{alpha_minmax}
\min\limits_{g(\cdot)}  - H^{\text{a}}_{\alpha}(Y|g(X,Y)) = \min\limits_{g(\cdot)} I^{\text{a}}_{\alpha}(g(X,Y);Y) - H_\alpha(Y),
\end{equation}
where $I^{\text{a}}_{\alpha}$ is the Arimoto mutual information and $H_\alpha$ is the R\'enyi entropy. Note that as $\alpha \to 1$, we recover the classical MI privacy setting and when $\alpha \to \infty$, we recover the $0$-$1$ loss.

\subsection{Data-driven GAP}
\label{sec:learningadversary}
So far, we have focused on a setting where the data holder has access to $P(X,Y)$. When $P(X,Y)$ is known, the data holder can simply solve the constrained minimax optimization problem in \eqref{eq:generalopt} (theoretical version of GAP) to obtain a privatization mechanism that would perform best against a chosen type of adversary. In the absence of $P(X,Y)$, we propose a data-driven version of GAP that allows the data holder to learn privatization mechanisms directly from a dataset of the form $\mathcal{D} = \{(x_{(i)},y_{(i)})\}_{i = 1}^{n}$. Under the data-driven version of GAP, we represent the privacy mechanism via a conditional generative model $g(X,Y;\theta_{p})$ parameterized by $\theta_{p}$. This generative model takes $(X,Y)$ as inputs and outputs $\hat{X}$. In the training phase, the data holder learns the optimal parameters $\theta_p$ by competing against a \textit{computational adversary}: a classifier modeled by a neural network $h(g(X,Y;\theta_{p}); \theta_a)$ parameterized by $\theta_{a}$. After convergence, we evaluate the performance of the learned $g(X,Y;\theta_{p}^{*})$ by computing the maximal probability of inferring $Y$ under the MAP adversary studied in the theoretical version of GAP.

We note that in theory, the functions $h$ and $g$ can (in general) be arbitrary; i.e., they can capture all possible learning algorithms. However, in practice, we need to restrict them to a rich hypothesis class. Figure~\ref{fig:ANN} shows an example of the GAP model in which the privatizer and adversary are modeled as multi-layer ``randomized'' neural networks. For a fixed $h$ and $g$, we quantify the adversary's {\textit{empirical loss}} using a continuous and differentiable function
\begin{equation}
\label{eq:lossemp}
L_{\text{EMP}}(\theta_p,\theta_a)= \frac{1}{n}\sum\limits_{i=1}^{n}\ell(h(g(x_{(i)},y_{(i)};\theta_{p});\theta_a),y_{(i)}),
\end{equation}
where $(x_{(i)}, y_{(i)})$ is the $i^{th}$ row of $\mathcal{D}$ and $\ell(h(g(x_{(i)},y_{(i)};\theta_{p});\theta_a),y_{(i)})$ is the adversary loss in the data-driven context. The optimal parameters for the privatizer and adversary are the solution to
\begin{align}
	\label{eq:learnedprivatizer}
	\min_{\theta_{p}}\max\limits_{\theta_a}\quad & -L_{\text{EMP}}(\theta_p,\theta_a) \\ \nonumber
	s.t.  \quad&  \mathbb{E}_{\mathcal{D}}[d(g(X,Y;\theta_{p}), X)]\le D,
\end{align}
where the expectation is taken over the dataset $\mathcal{D}$ and the randomness in $g$.

In keeping with the now common practice in machine learning, in the data-driven approach for GAP, one can use the empirical $\log$-loss function \citep{zhang2000neural, tang2013deep} given by \eqref{eq:lossemp} with
$$\ell(h(g(x_{(i)},y_{(i)};\theta_{p});\theta_a),y_{(i)})= -y_{(i)}\log h(g(x_{(i)},y_{(i)}; \theta_{p});\theta_{a})-(1-y_{(i)})\log(1-h(g(x_{(i)},y_{(i)}; \theta_{p});\theta_{a})),$$
which leads to a minimum cross-entropy adversary. As a result, the empirical loss of the adversary is quantified by the cross-entropy
\begin{equation}
\label{eq:lossCEbinary}
L_{\text{XE}}(\theta_p,\theta_a)= -\frac{1}{n}\sum\limits_{i=1}^{n}y_{(i)}\log h(g(x_{(i)},y_{(i)}; \theta_{p});\theta_{a})+(1-y_{(i)})\log(1-h(g(x_{(i)},y_{(i)}; \theta_{p});\theta_{a})).
\end{equation}


An alternative loss that can be readily used in this setting is the $\alpha$-loss introduced in Section \ref{sec:lossfunctions}. In the data-driven context, the $\alpha$-loss can be written as
\begin{multline}\label{alfaloss}
\ell(h(g(x_{(i)}, y_{(i)}; \theta_p); \theta_a), y_{(i)}) = \frac{\alpha}{\alpha  - 1} \left(y_{(i)} (1 - h(g(x_{(i)},y_{(i)}; \theta_{p});\theta_{a})^{1 - \frac{1}{\alpha}}) \right.  \\ \left. + (1 - y_{(i)})( 1- (1-h(g(x_{(i)},y_{(i)}; \theta_{p});\theta_{a}))^{1 - \frac{1}{\alpha}}) \right),
\end{multline}
for any constant $\alpha>1$. As discussed in Section \ref{sec:lossfunctions}, the $\alpha$-loss captures a variety of adversarial models and recovers both the $\log$-loss (when $\alpha \to 1$) and $0$-$1$ loss (when $\alpha \to \infty$). Futhermore, \eqref{alfaloss} suggests that $\alpha$-leakage can be used as a surrogate (and smoother) loss function for the $0$-$1$ loss (when $\alpha$ is relatively large).

\begin{figure}
	\centering
	\includegraphics[width=0.6\textwidth]{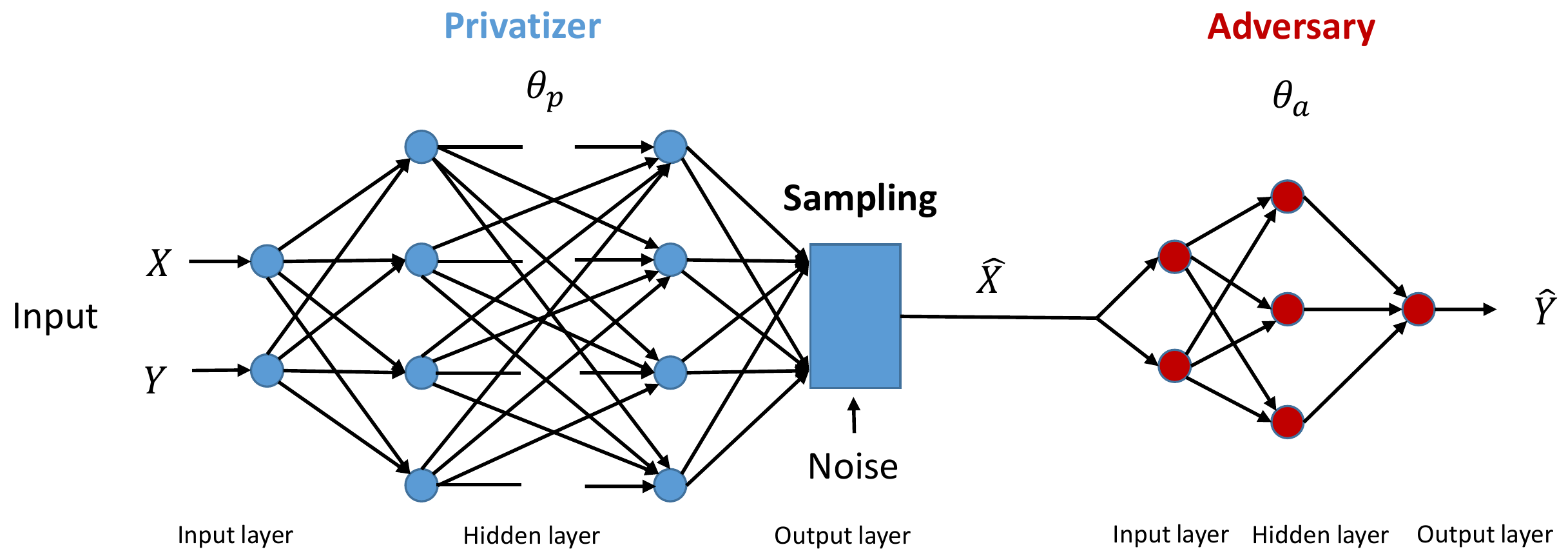}
	\caption{A multi-layer neural network model for the privatizer and adversary}
	\label{fig:ANN}
\end{figure}

The minimax optimization problem in \eqref{eq:learnedprivatizer} is a two-player non-cooperative game between the privatizer and the adversary. The strategies of the privatizer and adversary are given by $\theta_{p}$ and $\theta_{a}$, respectively. Each player chooses the strategy that optimizes its objective function \textit{w.r.t.}~what its opponent does. In particular, the privatizer must expect that if it chooses $\theta_{p}$, the adversary will choose a $\theta_{a}$ that {maximizes the negative of its own loss function} based on the choice of the privatizer. The optimal privacy mechanism is given by the equilibrium of the privatizer-adversary game.

In practice, we can learn the equilibrium of the game using an iterative algorithm presented in Algorithm \ref{alg:euclid}. We first {maximize the negative of the adversary's loss function} in the inner loop to compute the parameters of $h$ for a fixed $g$. Then, we minimize the privatizer's {loss function, which is modeled as the negative of the adversary's loss function,} to compute the parameters of $g$ for a fixed $h$. To avoid over-fitting and ensure convergence, we alternate between training the adversary for $k$ epochs and training the privatizer for one epoch. This results in the adversary moving towards its optimal solution for small perturbations of the privatizer \cite{goodfellow2014generative}.
\begin{algorithm}[ht]
	\caption{Alternating minimax privacy preserving algorithm}\label{alg:euclid}
	\begin{algorithmic}
		\State \textit{Input:} dataset $\mathcal{D}$, distortion parameter $D$, iteration number $T$\\
		\State \textit{Output:} Optimal privatizer parameter $\theta_{p}$\\
		\Procedure{Alernate Minimax}{$\mathcal{D}, D, T$}\\
		\State Initialize $\theta^1_{p}$ and $\theta^1_{a}$\\
		\For{$t=1,...,T$}\\
		\State Random minibatch of $M$ datapoints $\{x_{(1)},...,x_{(M)}\}$ drawn from full dataset\\
		\State Generate $\{\hat{x}_{(1)},...,\hat{x}_{(M)}\}$ via $\hat{x}_{(i)}=g(x_{(i)},y_{(i)};\theta^t_{p})$ \\
		\State Update the adversary parameter $\theta^{t+1}_a$ by stochastic gradient ascend for $k$ epochs
		{$$\quad \theta^{t+1}_a=\theta^{t}_a+\alpha_t\nabla_{\theta_a} \frac{1}{M}\sum\limits_{i=1}^{M}-\ell(h(\hat{x}_{(i)};\theta_{a}),y_{(i)}), \quad\alpha_t>0$$}
		\State Compute the descent direction $\nabla_{\theta_{p}} l(\theta_{p}, \theta^{t+1}_a)$, where
		{$$
		\nonumber
		\ell(\theta_{p},\theta^{t+1}_a)= -\frac{1}{M}\sum\limits_{i=1}^{M}\ell(h(g(x_{(i)},y_{(i)};\theta_{p});\theta^{t+1}_a),y_{(i)})
		$$}
		\quad\qquad subject to $ \frac{1}{M}\sum_{i=1}^{M}[d(g(x_{(i)},y_{(i)};\theta_{p}), x_{(i)})]\le D$\\
		\State Perform line search along $ \nabla_{\theta_{p}} l(\theta_{p}, \theta^{t+1}_a)$ and update
		$$\theta^{t+1}_{p}= \theta^{t}_{p}-\alpha_t \nabla_{\theta_{p}} \ell(\theta_{p}, \theta^{t+1}_a)$$
		\State Exit if solution converged\\
		\EndFor\\
		\State \textbf{return} $\theta^{t+1}_p$\\
		\EndProcedure\\
	\end{algorithmic}
\end{algorithm}

To incorporate the distortion constraint into the learning algorithm, we use the \textit{penalty method} \cite{lillo1993solving} and \textit{augmented Lagrangian method} \cite{eckstein2012augmented} to replace the constrained optimization problem by a series of unconstrained problems whose solutions asymptotically converge to the solution of the constrained problem. Under the penalty method, the unconstrained optimization problem is formed by adding a penalty to the objective function. The added penalty consists of a penalty parameter $\rho_t$ multiplied by a measure of violation of the constraint. The measure of violation is non-zero when the constraint is violated and is zero if the constraint is not violated. Therefore, in Algorithm \ref{alg:euclid}, the constrained optimization problem of the privatizer can be approximated by a series of unconstrained optimization problems with the loss function
{\begin{align}
\label{eq:penaltymethod}
\ell(\theta_{p},\theta^{t+1}_a)= &-\frac{1}{M}\sum\limits_{i=1}^{M}\ell(h(g(x_{(i)},y_{(i)};\theta_{p});\theta^{t+1}_a),y_{(i)})\\\nonumber&+\rho_t\max\{0, \frac{1}{M}\sum\limits_{i=1}^{M}d(g(x_{(i)},y_{(i)};\theta_{p}), x_{(i)})-D\},
\end{align}}where $\rho_t$ is a penalty coefficient which increases with the number of iterations $t$. For convex optimization problems, the solution to the series of unconstrained problems will eventually converge to the solution of the original constrained problem \cite{lillo1993solving}.

The augmented Lagrangian method is another approach to enforce equality constraints by penalizing the objective function whenever the constraints are not satisfied. Different from the penalty method, the augmented Lagrangian method combines the use of a Lagrange multiplier and a quadratic penalty term. Note that this method is designed for equality constraints. Therefore, we introduce a slack variable $\delta$ to convert the inequality distortion constraint into an equality constraint. Using the augmented Lagrangian method, the constrained optimization problem of the privatizer can be replaced by a series of unconstrained problems with the loss function given by
{\begin{align}
\label{eq:augmentedlagrange}
\ell(\theta_{p},\theta^{t+1}_a, \delta) = & -\frac{1}{M}\sum\limits_{i=1}^{M}\ell(h(g(x_{(i)},y_{(i)};\theta_{p});\theta^{t+1}_a),y_{(i)})\\\nonumber & +\frac{\rho_t}{2}(\frac{1}{M}\sum\limits_{i=1}^{M}d(g(x_{(i)},y_{(i)};\theta_{p}), x_{(i)})+\delta-D)^2\\\nonumber
& -\lambda_t (\frac{1}{M}\sum\limits_{i=1}^{M}d(g(x_{(i)},y_{(i)};\theta_{p}), x_{(i)})+\delta-D),
\end{align}}where $\rho_t$ is a penalty coefficient which increases with the number of iterations $t$ and $\lambda_t$ is updated according to the rule $\lambda_{t+1}=\lambda_t-\rho_t(\frac{1}{M}\sum\limits_{i=1}^{M}d(g(x_{(i)},y_{(i)};\theta_{p}), x_{(i)})+\delta-D)$. For convex optimization problems, the solution to the series of unconstrained problems formulated by the augmented Lagrangian method also converges to the solution of the original constrained problem \cite{eckstein2012augmented}.

\subsection{Our Focus}
Our GAP framework is very general and can be used to capture many notions of privacy via various decision rules and loss funcitons. In the rest of this paper, we investigate GAP under $0$-$1$ loss for two simple dataset models: (a) the binary data model (Section \ref{sec:binary}),  and (b) the binary Gaussian mixture model (Section \ref{sec:gaussian}). Under the binary data model, both $X$ and $Y$ are binary. Under the binary Gaussian mixture model, $Y$ is binary whereas $X$ is conditionally Gaussian. We use these results to validate that the data-driven version of GAP can discover ``theoretically optimal'' privatization schemes.

In the data-driven approach of GAP, since $P(X,Y$) is typically unknown in practice and our objective is to learn privatization schemes directly from data, we have to consider the empirical (data-driven) version of \eqref{eq:mapadversaryaccuracy}. Such an approach immediately hits a roadblock because taking derivatives of a {$0$-$1$ loss} function \textit{w.r.t.}~the parameters of $h$ and $g$ is ill-defined. To circumvent this issue, similar to the common practice in the ML literature, we use the empirical $\log$-loss (see Equation \eqref{eq:lossCEbinary}) as the loss function for the adversary. We derive game-theoretically optimal mechanisms for the $0$-$1$ loss function, and use them as a benchmark against which we compare the performance of the data-driven GAP mechanisms.

\section{Binary Data Model}
\label{sec:binary}
In this section, we study a setting where both the public and private variables are binary valued random variables. Let $p_{i,j}$ denote the joint probability of $(X, Y)=(i,j), $ where $i,j\in\{0,1\}$.
To prevent an adversary from correctly inferring the private variable $Y$ from the public variable $X$, the privatizer applies a randomized mechanism on $X$ to generate the privatized data $\hat{X}$. Since both the original and privatized public variables are binary, the distortion between $x$ and $\hat{x}$ can be quantified by the Hamming distortion; i.e. $d(\hat{x},x)=1$ if $\hat{x}\neq x$  and $d(\hat{x}, x)=0$ if $\hat{x}= x$. Thus, the expected distortion is given by $\mathbb{E}[d(\hat{X},X)]=P(\hat{X}\neq X)$.

\subsection{Theoretical Approach for Binary Data Model}
The adversary's objective is to correctly guess $Y$ from $\hat{X}$. We consider a MAP adversary who has access to the joint distribution of $(X,Y)$ and the privacy mechanism. The privatizer's goal is to privatize $X$ in a way that minimizes the adversary's probability of correctly inferring $Y$ from $\hat{X}$ subject to the distortion constraint. We first focus on private-data dependent (PDD) privacy mechanisms that depend on both $Y$ and $X$. We later consider private-data independent (PDI) privacy mechanisms that only depend on $X$. 

\subsubsection{PDD Privacy Mechanism}
Let $g(X,Y)$ denote a PDD mechanism. Since $X$, $Y$, and $\hat{X}$ are binary random variables, the mechanism $g(X,Y)$ can be represented by the conditional distribution $P(\hat{X}|X, Y)$ that maps the public and private variable pair $(X,Y)$ to an output $\hat{X}$ given by
$$P(\hat{X}=0|X=0,Y=0)=s_{0,0},\quad P(\hat{X}=0|X=0, Y=1)=s_{0,1},$$
$$P(\hat{X}=1|X=1,Y=0)=s_{1,0},\quad P(\hat{X}=1|X=1, Y=1)=s_{1,1}.$$
Thus, the marginal distribution of $\hat{X}$ is given by $$P(\hat{X}=0)=\sum\limits_{X,Y}P(\hat{X}=0|X, Y)P(X,Y)=s_{0,0}p_{0,0}+s_{0,1}p_{0,1}+(1-s_{1,0})p_{1,0}+(1-s_{1,1})p_{1,1},$$ $$P(\hat{X}=1)=\sum\limits_{X,Y}P(\hat{X}=1|X,Y)P(X,Y)=(1-s_{0,0})p_{0,0}+(1-s_{0,1})p_{0,1}+s_{1,0}p_{1,0}+s_{1,1}p_{1,1}.$$
If $\hat{X}=0$, the {adversary's inference accuracy} for guessing $\hat{Y}=1$ is
\begin{align}
P(Y=1, \hat{X}=0)=\sum\limits_{X}P(X, Y=1)P(\hat{X}=0|X, Y=1)=p_{1,1}(1-s_{1,1})
+p_{0,1}s_{0,1},
\end{align}
and the {inference accuracy for guessing} $\hat{Y}=0$ is
\begin{align}
P(Y=0, \hat{X}=0)=\sum\limits_{X}P(X, Y=0)P(\hat{X}=0|X, Y=0)=p_{1,0}(1-s_{1,0})+p_{0,0}s_{0,0}.
\end{align}
Let $\mathbf{s}=\{s_{0,0},s_{0,1},s_{1,0},s_{1,1}\}$. For $\hat{X}=0$, the MAP adversary's inference accuracy is given by
\begin{align}
\label{eq:LZ0xy}
{P^{\text{(B)}}_{\text{d}}(\mathbf{s},\hat{X}=0)}&=\max\{P(Y=1, \hat{X}=0),P(Y=0, \hat{X}=0)\}.
\end{align}
Similarly, if $\hat{X}=1$, the MAP adversary's inference accuracy is given by
\begin{align}
\label{eq:LZ1xy}
{P^{\text{(B)}}_{\text{d}}(\mathbf{s},\hat{X}=1)}& =\max\{P(Y=1, \hat{X}=1),P(Y=0, \hat{X}=1)\},
\end{align} where
\begin{align}
& P(Y=1, \hat{X}=1)=\sum\limits_{X}P(X, Y=1)P(\hat{X}=1|X, Y=1)=p_{1,1}s_{1,1}
+p_{0,1}(1-s_{0,1}),\\\nonumber
&P(Y=0, \hat{X}=1)=\sum\limits_{X}P(X, Y=0)P(\hat{X}=1|X, Y=0)=p_{1,0}s_{1,0}+p_{0,0}(1-s_{0,0}).
\end{align}
As a result, for a fixed privacy mechanism $\mathbf{s}$, the MAP adversary's inference accuracy can be written as $${P^{\text{(B)}}_{\text{d}}=\max\limits_{h(\cdot)}P(h(g(X,Y))=Y)= P^{\text{(B)}}_{\text{d}}(\mathbf{s},\hat{X}=0)+P^{\text{(B)}}_{\text{d}}(\mathbf{s},\hat{X}=1)}.$$ Thus, the optimal PDD privacy mechanism is determined by solving
\begin{align}
\label{eq:toyprivatizerlossxy}
\min_{\mathbf{s}}\quad & {P^{\text{(B)}}_{\text{d}}(\mathbf{s},\hat{X}=0)+P^{\text{(B)}}_{\text{d}}(\mathbf{s},\hat{X}=1)}\\\nonumber
s.t.  \quad & P(\hat{X}=0,X=1)+P(\hat{X}=1,X=0)\le D\\\nonumber
& \mathbf{s}\in[0,1]^4.
\end{align}

Notice that the above constrained optimization problem is a four dimensional optimization problem parameterized by $\mathbf{p}=\{p_{0,0},p_{0,1},p_{1,0},p_{1,1}\}$ and $D$. Interestingly, we can formulate \eqref{eq:toyprivatizerlossxy} as a linear program (LP) given by
\begin{align}
\label{eq:toyprivatizerlossxylp}
\min\limits_{s_{1,1},s_{0,1},s_{1,0},s_{0,0}, t_0, t_1}\quad& t_0+t_1\\\nonumber
s.t.\quad\qquad  & 0\le s_{1,1},s_{0,1},s_{1,0},s_{0,0}\le 1 \\\nonumber
& p_{1,1}(1-s_{1,1})+p_{0,1}s_{0,1}\le t_0\\\nonumber
& p_{1,0}(1-s_{1,0})+p_{0,0}s_{0,0}\le t_0\\\nonumber
& p_{1,1}s_{1,1}+p_{0,1}(1-s_{0,1})\le t_1\\\nonumber
& p_{1,0}s_{1,0}+p_{0,0}(1-s_{0,0})\le t_1\\\nonumber	
&  p_{1,1}(1-s_{1,1})+p_{0,1}(1-s_{0,1})+p_{1,0}(1-s_{1,0})+p_{0,0}(1-s_{0,0})\le D,
\end{align}
where $t_0$ and $t_1$ are two slack variables representing the maxima in \eqref{eq:LZ0xy} and \eqref{eq:LZ1xy}, respectively. The optimal mechanism can be obtained by numerically solving \eqref{eq:toyprivatizerlossxylp} using any off-the-shelf LP solver.
\subsubsection{PDI Privacy Mechanism}
\label{sec:PDIP}
In the previous section, we considered PDD privacy mechanisms. Although we were able to formulate the problem as a linear program with four variables, determining a closed form solution for such a highly parameterized problem is not analytically tractable. Thus, we now consider the simple (yet meaningful) class of PDI privacy mechanisms. Under PDI privacy mechanisms, the Markov chain $Y\rightarrow X\rightarrow \hat{X}$ holds. As a result, $P(Y,\hat{X}=\hat{x})$ can be written as
\begin{align}
P(Y,\hat{X}=\hat{x})&=\sum\limits_{X}P(Y,\hat{X}=\hat{x}|X)P(X)\\
&=\sum\limits_{X}P(Y|X)P(\hat{X}=\hat{x}|X)P(X)\\
&=\sum\limits_{X}P(Y,X)P(\hat{X}=\hat{x}|X),
\end{align}
where the second equality is due to the conditional independence property of the Markov chain $Y\rightarrow X\rightarrow \hat{X}$.

For the PDI mechanisms, the privacy mechanism $g(X, Y)$ can be represented by the conditional distribution $P(\hat{X}|X)$. To make the problem more tractable, we focus on a slightly simpler setting in which $Y=X\oplus N$, where $N\in \{0,1\}$ is a random variable independent of $X$ and follows a Bernoulli distribution with parameter $q$. In this setting, the joint distribution of $(X,Y)$ can be computed as
\begin{align}
P(X=1,Y=1)=P(Y=1|X=1)P(X=1)& =p(1-q),\\
P(X=0,Y=1)=P(Y=1|X=0)P(X=0)& =(1-p)q,\\
P(X=1,Y=0)=P(Y=0|X=1)P(X=1)& =pq,\\
P(X=0,Y=0)=P(Y=0|X=0)P(X=0)& =(1-p)(1-q).
\end{align}

Let $\mathbf{s}=\{s_0,s_1\}$ in which $s_0=P(\hat{X}=0|X=0)$ and $s_1=P(\hat{X}=1|X=1)$. The joint distribution of $(Y,\hat{X})$ is given by
\begin{align}
\nonumber
P(Y=1, \hat{X}=0)=p(1-q)(1-s_1)
+(1-p)qs_0, \\\nonumber P(Y=0, \hat{X}=0)=pq(1-s_1)+(1-p)(1-q)s_0,\\\nonumber P(Y=1, \hat{X}=1)=p(1-q)s_1+(1-p)q(1-s_0), \\\nonumber P(Y=0, \hat{X}=1)=pqs_1+(1-p)(1-q)(1-s_0).
\end{align}
Using the above joint probabilities, for a fixed $\mathbf{s}$, we can write the {MAP adversary's inference accuracy as
\begin{align}
\label{eq:mappdi}
 P^{\text{(B)}}_{\text{d}}=\max_{h(\cdot)}P(h(g(X,Y))=Y)& = \max\{P(Y=1, \hat{X}=0),P(Y=0, \hat{X}=0)\}\\\nonumber
 & +\max\{P(Y=1, \hat{X}=1),P(Y=0, \hat{X}=1)\}.
\end{align}
Therefore, the optimal PDI privacy mechanism is given by the solution to
\begin{align}
\label{eq:optprivatizer}
\min\limits_{\mathbf{s}}\quad & P^{\text{(B)}}_{\text{d}}\\\nonumber
s.t.  \quad &  P(\hat{X}=0,X=1)+P(\hat{X}=1,X=0)\le D\\\nonumber
& \mathbf{s}\in[0,1]^2,
\end{align}
where the distortion in \eqref{eq:optprivatizer} is given by $(1-s_0)(1-p)+(1-s_1)p$. By \eqref{eq:mappdi}, $P^{\text{(B)}}_{\text{d}}$} can be considered as a sum of two functions, where each function is a maximum of two linear functions. Therefore, it is convex in $s_0$ and $s_1$ for different values of $p,q$ and $D$.


\begin{Theorem}
	\label{thm:binary}
	For fixed $p,q$ and $D$, there exists infinitely many PDI privacy mechanisms that achieve the optimal privacy-utility tradeoff. If $q=\frac{1}{2}$, any privacy mechanism that satisfies $\{s_0,s_1| ps_1+(1-p)s_0\ge 1-D, s_0, s_1\in[0,1]\}$ is optimal. If $q\neq\frac{1}{2}$, the optimal PDI privacy mechanism is given as follows:
	\begin{itemize}
		\item If $1-D >\max\{p,1-p\}$, the optimal privacy mechanism is given by $\{s_0,s_1| ps_1+(1-p)s_0=1-D, s_0, s_1\in[0,1]\}$. The adversary's accuracy of correctly guessing the private variable is
		\begin{align}
		\left\{
		\begin{array}{ll}
		(1-2q)(1-D)+q  & \mbox{if } q<\frac{1}{2} \\
		(2q-1)(1-D)+1-q & \mbox{if } q>\frac{1}{2}
		\end{array}
		\right..
		\end{align}
		\item Otherwise, the optimal privacy mechanism is given by
		$\{s_0,s_1| \max\{\min\{p,1-p\},1-D\} \le ps_1+(1-p)s_0\le\max\{p,1-p\}, s_0,s_1\in[0,1]\}$ and the adversary's accuracy of correctly guessing the private variable is
		\begin{align}
		\left\{
		\begin{array}{ll}
		p(1-q)+(1-p)q & \mbox{if } p\ge\frac{1}{2}, q<\frac{1}{2} \mbox{ or } p\le\frac{1}{2}, q>\frac{1}{2}\\
		pq+(1-p)(1-q) & \mbox{if } p\ge\frac{1}{2}, q>\frac{1}{2} \mbox{ or } p\le\frac{1}{2}, q<\frac{1}{2}
		\end{array}
		\right..
		\end{align}
	\end{itemize}
\end{Theorem}
\textit{Proof sketch:} The proof of Theorem~\ref{thm:binary} is provided in Appendix~\ref{binaryproof}. We briefly sketch the proof details here. For the special case $q=\frac{1}{2}$, the solution is trivial since the private variable $Y$ is independent of the public variable $X$. Thus, the optimal solution is given by any $s_0$, $s_1$ that satisfies the distortion constraint $\{s_0,s_1| ps_1+(1-p)s_0\ge 1-D, s_0, s_1\in[0,1]\}$. For $q\neq\frac{1}{2}$, we separate the optimization problem in \eqref{eq:optprivatizer} into four subproblems based on the decision of the adversary. We then compute the optimal privacy mechanism of the privatizer in each subproblem. Summarizing the optimal solutions to the subproblems for different values of $p,q$ and $D$ yields Theorem \ref{thm:binary}.

\textit{Remark:} Note that if $1-D >\max\{p,1-p\}$, i.e., $D <\min\{p,1-p\}$, the privacy guarantee achieved by the optimal PDI mechanism (the MAP adversary's accuracy of correctly guessing the private variable) decreases linearly with $D$. For $D \ge \min\{p,1-p\}$, the optimal PDI mechanism achieves a constant privacy guarantee regardless of $D$. However, in this case, the privatizer can just use the optimal privacy mechanism with $D=\min\{p,1-p\}$ to optimize privacy guarantee without further sacrificing utility.

\subsection{Data-driven Approach for Binary Data Model}
In practice, the joint distribution of $(X,Y)$ is often unknown to the data holder. Instead, the data holder has access to a dataset $\mathcal{D}$, which is used to learn a good privatization mechanism in a generative adversarial fashion. In the training phase, the data holder learns the parameters of the conditional generative model (representing the privatization scheme) by competing against a computational adversary represented by a neural network. The details of both neural networks are provided later in this section. When convergence is reached, we evaluate the performance of the learned privatization scheme by computing the accuracy of inferring $Y$ under a strong MAP adversary that: (a) has access to the joint distribution of $(X,Y)$, (b) has knowledge of the learned privacy mechanism, and (c) can compute the MAP rule. Ultimately, the data holder's hope is to learn a privatization scheme that matches the one obtained under the game-theoretic framework, where both the adversary and privatizer are assumed to have access to $P(X,Y)$. To evaluate our data-driven approach, we compare the mechanisms learned in an adversarial fashion on $\mathcal{D}$ with the game-theoretically optimal ones.

Since the private variable $Y$ is binary, we use the {empirical log-loss function for the adversary} (see Equation \eqref{eq:lossCEbinary}).
For a fixed $\theta_{p}$, the adversary learns the optimal $\theta^*_a$ by maximizing {$-L_{\text{XE}}(h(g(X,Y;\theta_{p});\theta_{a}),Y)$} given in Equation \eqref{eq:lossCEbinary}. For a fixed $\theta_a$, the privatizer learns the optimal $\theta_p^*$ by minimizing {$-L_{\text{XE}}(h(g(X,Y;\theta_{p});\theta_{a}),Y)$} subject to the distortion constraint (see Equation \eqref{eq:learnedprivatizer}).
Since both $X$ and $Y$ are binary variables, we can use the privatizer parameter $\theta_{p}$ to represent the privacy mechanism $\mathbf{s}$ directly. For the adversary, we define $\theta_{a}=(\theta_{a,0}, \theta_{a,1})$, where $\theta_{a,0}=P(Y=0|\hat{X}=0)$ and $\theta_{a,1}=P(Y=1|\hat{X}=1)$. Thus, given a privatized public variable input $g(x_{(i)},y_{(i)};\theta_{p})\in\{0,1\}$, the output belief of the adversary guessing $y_{(i)}=1$ can be written as $(1-\theta_{a,0})(1-g(x_{(i)},y_{(i)};\theta_{p}))+\theta_{a,1}g(x_{(i)},y_{(i)};\theta_{p})$.

\begin{figure}
	\centering
	\includegraphics[width=0.8\textwidth]{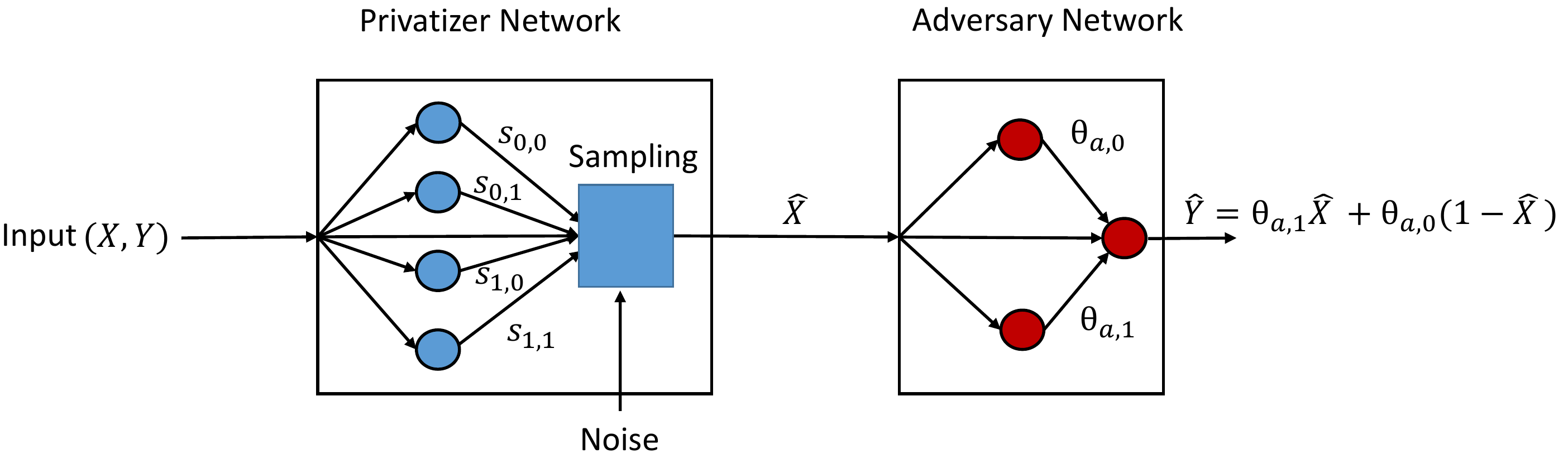}	
	\caption{Neural network structure of the privatizer and adversary for binary data model}
	\label{fig:nnbinary}
\end{figure}

For PDD privacy mechanisms, we have $\theta_{p}=\mathbf{s}=\{s_{0,0},s_{0,1},s_{1,0},s_{1,1}\}$. Given the fact that both $x_{(i)}$ and $y_{(i)}$ are binary, we use two simple neural networks to model the privatizer and the adversary. As shown in Figure \ref{fig:nnbinary}, the privatizer is modeled as a two-layer neural network parameterized by $\mathbf{s}$, while the adversary is modeled as a two-layer neural network classifier. From the perspective of the privatizer, the belief of an adversary guessing $y_{(i)}=1$ conditioned on the input $(x_{(i)},y_{(i)})$ is given by
\begin{align}
h(g(x_{(i)},y_{(i)}; \mathbf{s});\theta_{a})& =\theta_{a,1}
P(\hat{x}_{(i)}=1)+(1-\theta_{a,0})P(\hat{x}_{(i)}=0),
\end{align}
where
\begin{align}
\nonumber
P(\hat{x}_{(i)}=1)=& x_{(i)}y_{(i)}s_{1,1}+(1-x_{(i)})y_{(i)}(1-s_{0,1})\\\nonumber & +x_{(i)}(1-y_{(i)})s_{1,0}+(1-x_{(i)})(1-y_{(i)})(1-s_{0,0}),\\\nonumber
P(\hat{x}_{(i)}=0)=& x_{(i)}y_{(i)}(1-s_{1,1})+(1-x_{(i)})y_{(i)}s_{0,1}\\\nonumber &+x_{(i)}(1-y_{(i)})(1-s_{1,0})+(1-x_{(i)})(1-y_{(i)})s_{0,0}.
\end{align}
Furthermore, the expected distortion is given by
\begin{align}
\mathbb{E}_{\mathcal{D}}[d(g(X,Y;\mathbf{s}), X)]=& \frac{1}{n}\sum\limits_{i=1}^{n} [x_{(i)}y_{(i)}(1-s_{1,1})+x_{(i)}(1-y_{(i)})(1-s_{1,0})\\\nonumber
& +(1-x_{(i)})y_{(i)}(1-s_{0,1})+(1-x_{(i)})(1-y_{(i)})(1-s_{0,0})].
\end{align}
Similar to the PDD case, we can also compute the belief of guessing $y_{(i)}=1$ conditional on the input $(x_{(i)},y_{(i)})$ for the PDI schemes. Observe that in the PDI case, $\theta_p=\mathbf{s}=\{s_0,s_1\}$. Therefore, we have
\begin{align}
h(g(x_{(i)},y_{(i)};\mathbf{s} );\theta_{a})=\theta_{a,1}[x_{(i)}s_{1}+(1-x_{(i)})(1-s_{0})]+(1-\theta_{a,0})[(1-x_{(i)})s_{0}+x_{(i)}(1-s_{1})].
\end{align}
Under PDI schemes, the expected distortion is given by
\begin{align}
\mathbb{E}_{\mathcal{D}}[d(g(X,Y;\mathbf{s}), X)]=\frac{1}{n}\sum\limits_{i=1}^{n}[x_{(i)}(1-s_{1})+(1-x_{(i)})(1-s_{0})].
\end{align}
Thus, we can use Algorithm \ref{alg:euclid} proposed in Section \ref{sec:learningadversary} to learn the optimal PDD and PDI privacy mechanisms from the dataset.

\subsection{Illustration of Results}
We now evaluate our proposed GAP framework using synthetic datasets. We focus on the setting in which $Y=X\oplus N$, where $N\in \{0,1\}$ is a random variable independent of $X$ and follows a Bernoulli distribution with parameter $q$. We generate two synthetic datasets with $(p,q)$ equal to $(0.75,0.25)$ and $(0.5,0.25)$, respectively. Each synthetic dataset used in this experiment contains $10,000$ training samples and $2,000$ test samples. We use Tensorflow \citep{abadi2016tensorflow} to train both the privatizer and the adversary using Adam optimizer with a learning rate of $0.01$ and a minibatch size of $200$. The distortion constraint is enforced by the penalty method provided in \eqref{eq:penaltymethod}. 

\begin{figure}[H]
	\centering
	\begin{subfigure}[t]{0.45\textwidth}
		\includegraphics[width=0.98\columnwidth]{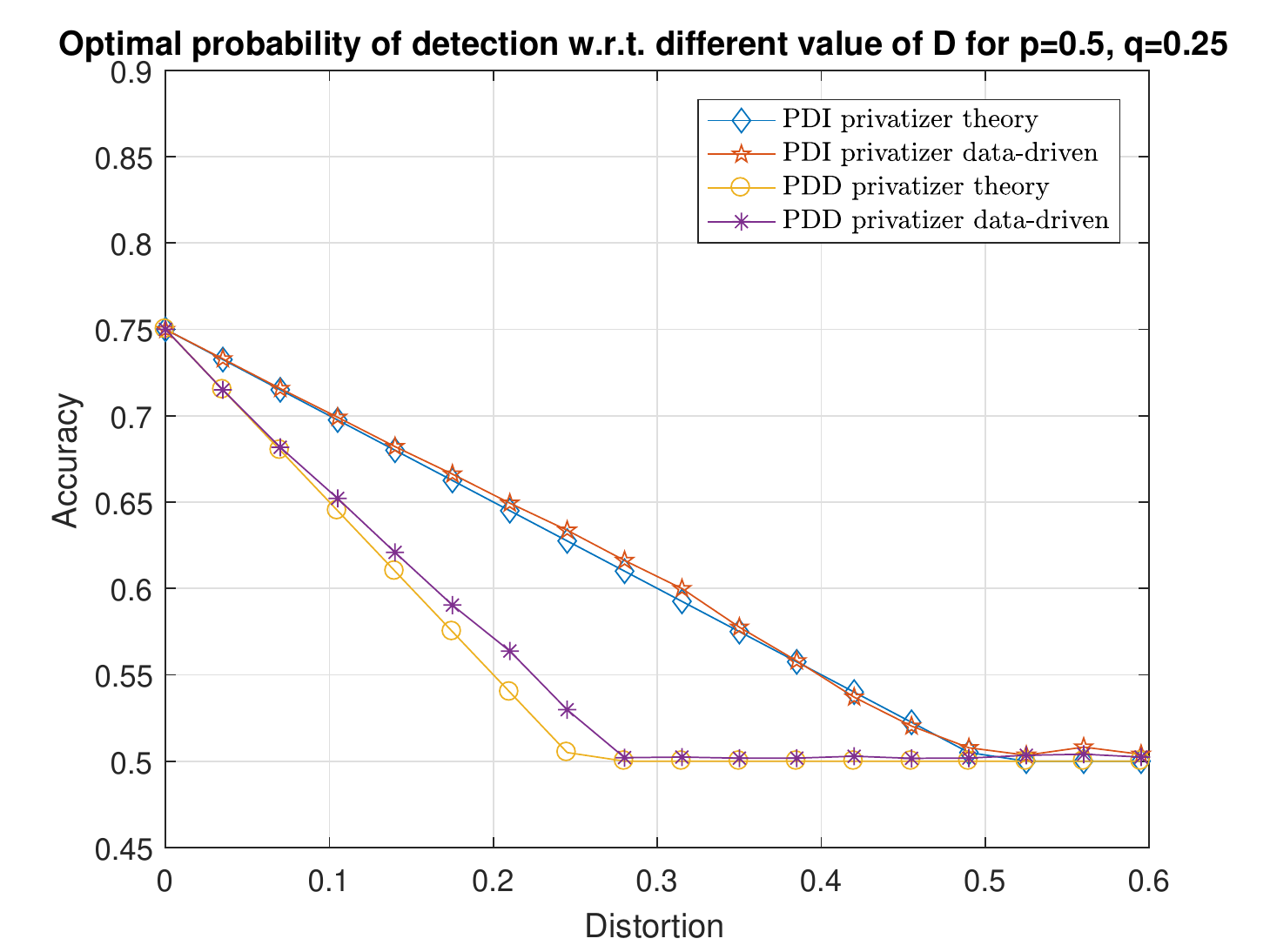}
		\caption{Performance of privacy mechanisms against MAP adversary for $p=0.5$}
		\label{fig:tl05}
	\end{subfigure}\qquad
	\begin{subfigure}[t]{0.45\textwidth}
		\includegraphics[width=0.98\columnwidth]{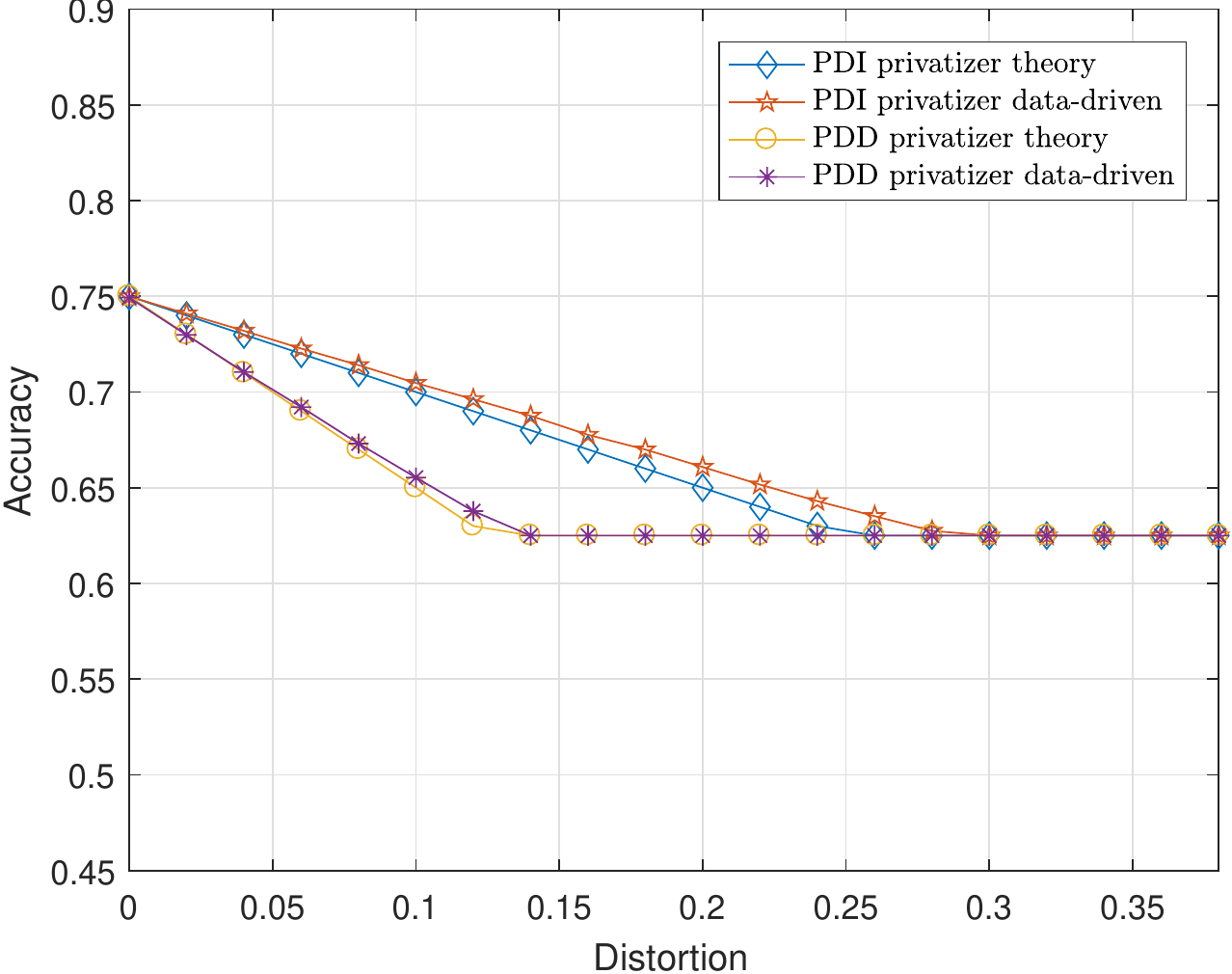}
		\caption{Performance of privacy mechanisms against MAP adversary for $p=0.75$}
		\label{fig:tl075}
	\end{subfigure}\\
	\begin{subfigure}[t]{0.45\textwidth}
		\includegraphics[width=0.98\columnwidth]{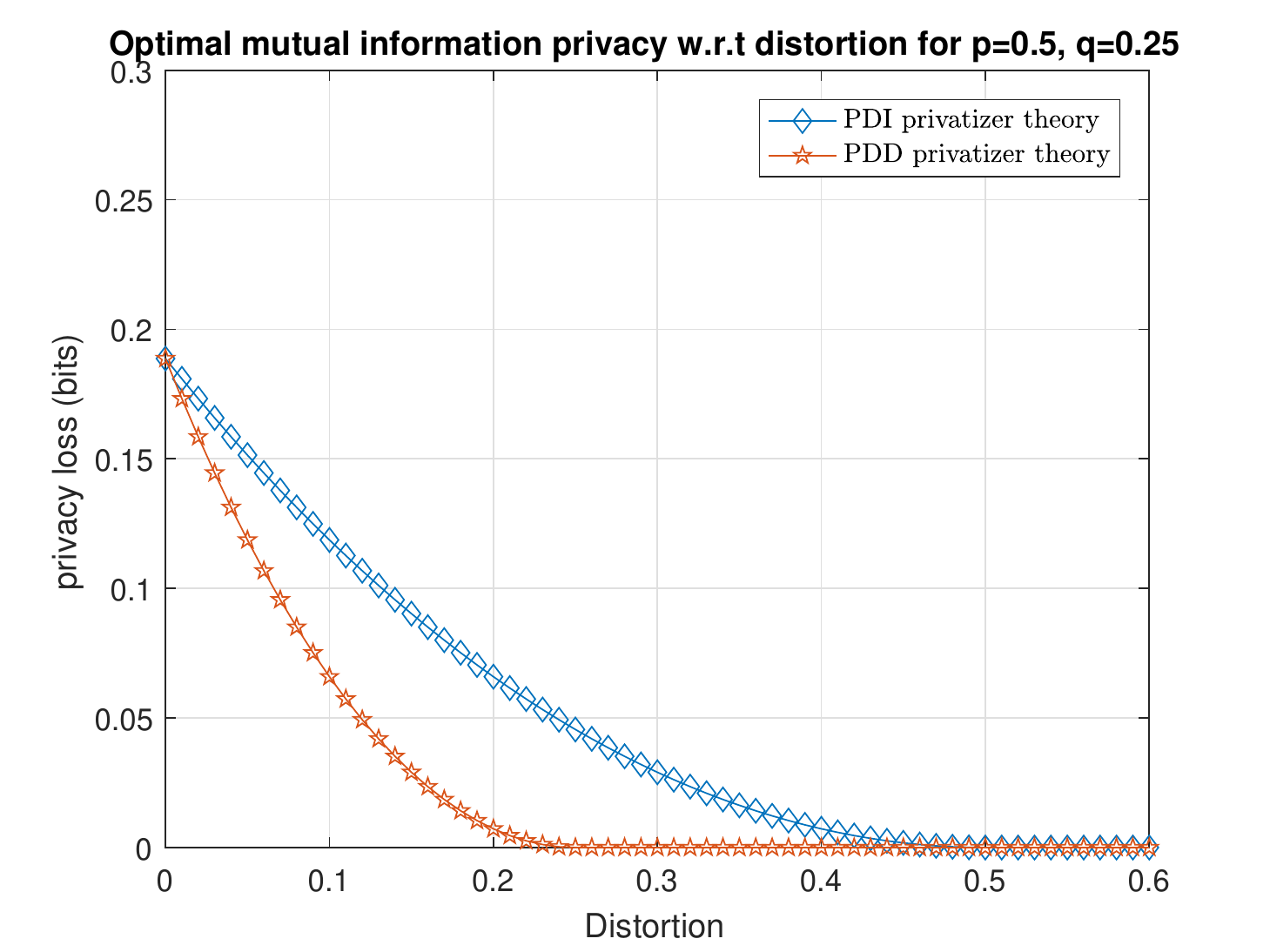}
		\caption{Performance of privacy mechanisms under MI privacy metric for $p=0.5$}
		\label{fig:mi05}
	\end{subfigure}\qquad
	\begin{subfigure}[t]{0.45\textwidth}
		\includegraphics[width=0.98\columnwidth]{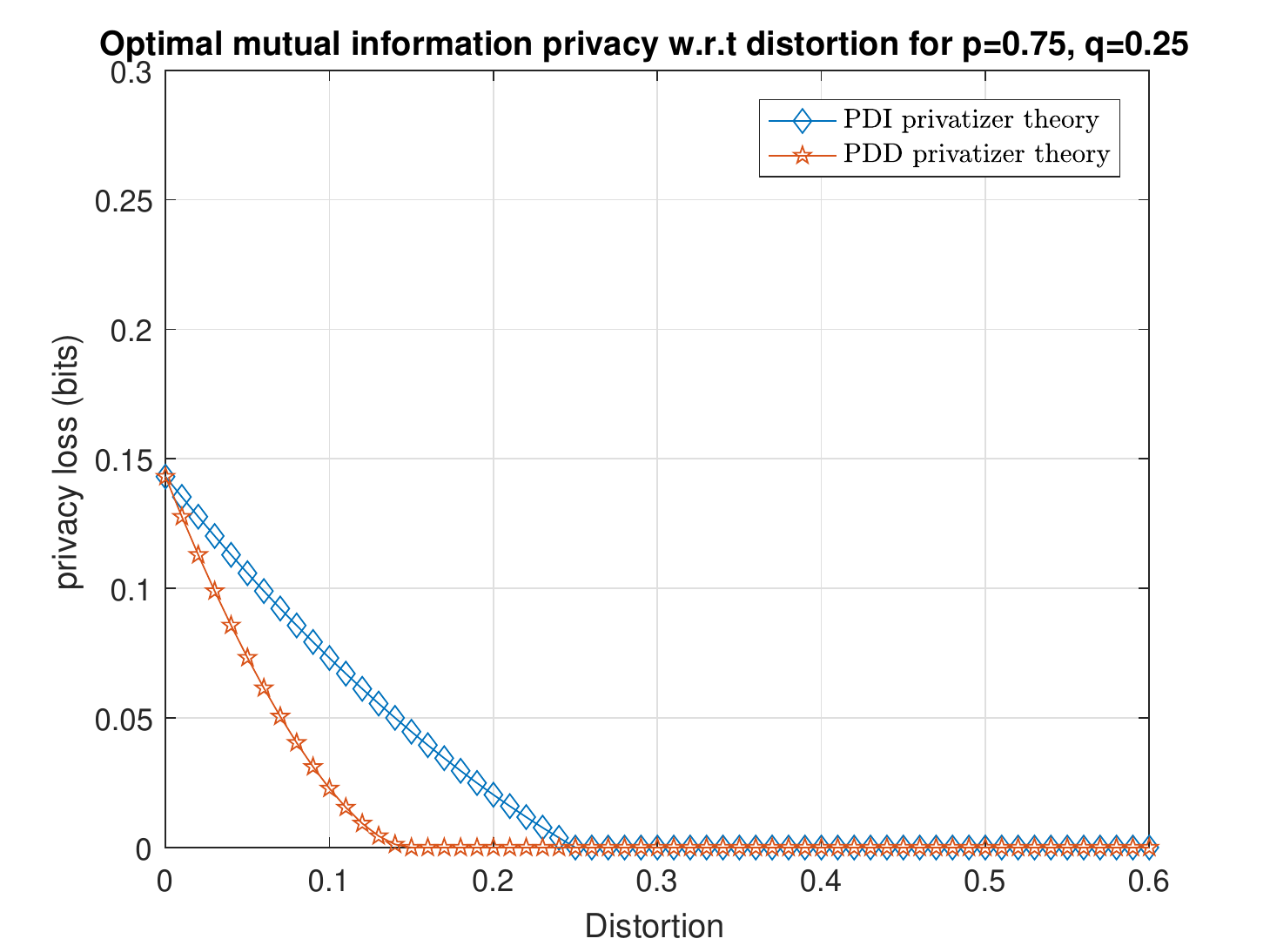}
		\caption{Performance of privacy mechanisms under MI privacy metric for $p=0.75$}
		\label{fig:mi075}
	\end{subfigure}
	\caption{Privacy-distortion tradeoff for binary data model}%
	\label{fig:binaryperformancetl}
\end{figure}

Figure \ref{fig:tl05} illustrates the performance of both optimal PDD and PDI privacy mechanisms against a strong theoretical MAP adversary when $(p, q)=(0.5, 0.25)$. It can be seen that the {inference} accuracy of the MAP adversary reduces as the distortion increases for both optimal PDD and PDI privacy mechanisms. As one would expect, the PDD privacy mechanism achieves a lower {inference} accuracy for the adversary, i.e., better privacy, than the PDI mechanism. Furthermore, when the distortion is higher than some threshold, the {inference} accuracy of the MAP adversary saturates regardless of the distortion. This is due to the fact that the correlation between the private variable and the privatized public variable cannot be further reduced once the distortion is larger than the saturation threshold. Therefore, increasing distortion will not further reduce the accuracy of the MAP adversary. We also observe that the privacy mechanism obtained via the data-driven approach performs very well when pitted against the MAP adversary (maximum accuracy difference around $3 \%$ compared to the theoretical approach). In other words, for the binary data model, the data-driven version of GAP can yield privacy mechanisms that perform as well as the mechanisms computed under the theoretical version of GAP, which assumes that the privatizer has access to the underlying distribution of the dataset.

Figure \ref{fig:tl075} shows the performance of both optimal PDD and PDI privacy mechanisms against the MAP adversary for $(p, q)=(0.75,0.25)$. Similar to the equal prior case, we observe that both PDD and PDI privacy mechanisms reduce the accuracy of the MAP adversary as the distortion increases and saturate when the distortion goes above a certain threshold. It can be seen that the saturation thresholds for both PDD and PDI privacy mechanisms in Figure \ref{fig:tl075} are lower than the ``equal prior'' case plotted in Figure \ref{fig:tl05}. The reason is that when $(p, q)=(0.75,0.25)$, the correlation between $Y$ and $X$ is weaker than the ``equal prior'' case. Therefore, it requires less distortion to achieve the same privacy. We also observe that the performance of the GAP mechanism obtained via the data-driven approach is comparable to the mechanism computed via the theoretical approach.

The performance of the GAP mechanism obtained using {the $\log$-loss function} (i.e., MI privacy) is plotted in Figure \ref{fig:mi05} and \ref{fig:mi075}. Similar to the MAP adversary case, as the distortion increases, the mutual information between the private variable and the privatized public variable achieved by the optimal PDD and PDI mechanisms decreases as long as the distortion is below some threshold. When the distortion goes above the threshold, the optimal privacy mechanism is able to make the private variable and the privatized public variable independent regardless of the distortion. Furthermore, the values of the saturation thresholds are very close to what we observe in Figure \ref{fig:tl05} and \ref{fig:tl075}.

\section{Binary Gaussian Mixture Model}
\label{sec:gaussian}

Thus far, we have studied a simple binary dataset model. In many real datasets, the sample space of variables often takes more than just two possible values. It is well known that the Gaussian distribution is a  flexible approximate for many distributions \citep{weisstein2002normal}. Therefore, in this section, we study a setting where $Y \in \{0, 1\}$ and ${X}$ is a Gaussian random variable whose mean and variance are dependent on $Y$. Without loss of generality, let $\mathbb{E}[X|Y = 1] = - \mathbb{E}[X|Y= 0]=\mu$ and $P(Y =1) = \tilde{p}$. Thus, $X|Y =0 \sim \mathcal{N}(-\mu, \sigma_0^2)$  and $X|Y =1 \sim \mathcal{N}(\mu, \sigma_1^2)$.

Similar to the binary data model, we study two privatization schemes: (a) private-data independent (PDI) schemes (where $\hat{X} = g(X)$), and (b) private-data dependent (PDD) schemes (where $\hat{X} = g(X, Y)$). In order to have a tractable model for the privatizer, we assume $g(X,Y)$ is realized by adding an affine function of an independently generated random noise to the public variable $X$. The affine function enables controlling both the mean and variance of the privatized data. In particular,  we consider $g(X,Y)= X + (1-Y)\beta_0 - Y\beta_1 + (1-Y)\gamma_0N + Y \gamma_1N$, in which $N$ is a one dimensional random variable and $\beta_0 ,\beta_1,\gamma_0,\gamma_1$ are constant parameters. The goal of the privatizer is to sanitze the public data $X$ subject to the distortion constraint $\mathbb{E}_{\hat{X},X}||\hat{X}-X||_2^2\le D$. 

\subsection{Theoretical Approach for Binary Gaussian Mixture Model}
We now investigate the theoretical approach under which both the privatizer and the adversary have access to $P(X,Y)$. To make the problem more tractable, let us consider a slightly simpler setting in which $\sigma_0 = \sigma_1 = \sigma$. We will relax this assumption later when we take a data-driven approach. We further assume that $N$ is a standard Gaussian random variable. One might, rightfully, question our choice of focusing on adding (potentially $Y$-dependent) Gaussian noise. Though other distributions can be considered, our approach is motivated by the following two reasons:
\begin{itemize}
	\item (a) Even though it is known that adding Gaussian noise is not the worst case noise adding mechanism for non-Gaussian $X$ \citep{shamai1992worst}, identifying the optimal noise distribution is mathematically intractable. Thus, for tractability and ease of analysis, we choose Gaussian noise.
	\item (b) Adding Gaussian noise to each data entry preserves the conditional Gaussianity of the released dataset.
\end{itemize}
In what follows, we will analyze a variety of PDI and PDD mechanisms.
\subsubsection{PDI Gaussian Noise Adding Privacy Mechanism}
We consider a PDI noise adding privatization scheme which adds an affine function of the standard Gaussian noise to the public variable. Since the privacy mechanism is PDI, we have $g(X,Y)=X + \beta +\gamma N$, where $\beta$ and $\gamma$ are constant parameters and $N\sim\mathcal{N}(0,1)$. Using the classical Gaussian hypothesis testing analysis \cite{van2004detection}, it is straightforward to verify that the {optimal inference accuracy (i.e., probability of detection) of the MAP adversary is given by
\begin{equation}
\label{eq:gaussianscheme0}
P^{\text{(G)}}_{\text{d}}=\tilde{p} Q\left(-  {\frac{\alpha}{2} }  + {\frac{1}{\alpha}} \ln\left( { \frac{1  - \tilde{p}}{\tilde{p}}} \right) \right) + (1 - \tilde{p}) Q\left(-  {\frac{\alpha}{2}}  - {\frac{1}{\alpha}} \ln\left( {\frac{1  - \tilde{p}}{\tilde{p}}} \right) \right),
\end{equation}}
where $\alpha=\frac{2\mu}{\sqrt{\gamma^2+\sigma^2}}$ and $Q(x)=\frac{1}{\sqrt{2\pi}}\int_{x}^{\infty}\exp(-\frac{u^2}{2}) du$. Moreover, since $\mathbb{E}_{\hat{X},X}[ d(\hat{X}, X)] = \beta^2+\gamma^2$, the distortion constraint is equivalent to $\beta^2+\gamma^2 \leq D$. 
\begin{Theorem}
	\label{thm:gaussianscheme0}
	For a PDI Gaussian noise adding privatization scheme given by  $g(X,Y)=X +\beta+\gamma N $, with $\beta \in \mathbb{R}$ and $\gamma\ge0$, the optimal parameters are given by
	\begin{align}
	\beta^*=0,\gamma^*=\sqrt{D}.
	\end{align}
	Let $\alpha^*=\frac{2\mu}{\sqrt{D+\sigma^2}}$. For this optimal scheme, the accuracy of the MAP adversary is
	\begin{align}
	{P^{\text{(G)*}}_{\text{d}}}=\tilde{p} Q\left(-  {\frac{\alpha^*}{2} }  + {\frac{1}{\alpha^*}} \ln\left( { \frac{1  - \tilde{p}}{\tilde{p}}} \right) \right) + (1 - \tilde{p}) Q\left(-  {\frac{\alpha^*}{2}}  - {\frac{1}{\alpha^*}} \ln\left( {\frac{1  - \tilde{p}}{\tilde{p}}} \right) \right).
	\end{align}
\end{Theorem}
The proof of Theorem \ref{thm:gaussianscheme0} is provided in Appendix \ref{gaussianscheme0proof}. We observe that the PDI Gaussian noise adding privatization scheme which minimizes the inference accuracy of the MAP adversary with distortion upper-bounded by $D$ is to add a zero-mean Gaussian noise with variance $D$.


\subsubsection{PDD Gaussian Noise Adding Privacy Mechanism}	
\label{sec:pddgaussian}
For PDD privatization schemes, we first consider a simple case in which $\gamma_0=\gamma_1=0$. Without loss of generality, we assume that both $\beta_0$ and $\beta_1$ are non-negative. The privatized data is given by $\hat{X} = X + (1- Y)\beta_0 - Y\beta_1$. This is a PDD mechanism since $\hat{X}$ depends on both $X$ and $Y$. Intuitively, this mechanism privatizes the data by shifting the two Gaussian distributions (under $Y = 0$ and $Y = 1$) closer to each other. Under this mechanism, it is easy to show that the {adversary's MAP probability of inferring the private variable $Y$ from $\hat{X}$ is given by $P^{\text{(G)}}_{\text{d}}$} in \eqref{eq:gaussianscheme0}
with $\alpha = { \frac{2\mu  -(\beta_1 + \beta_0)}{\sigma} }$. Observe that since $d(\hat{X}, X) = ((1- Y)\beta_0 - Y\beta_1)^2$, we have $\mathbb{E}_{\hat{X},X}[ d(\hat{X}, X)] = (1-\tilde{p})\beta_0^2 + \tilde{p}\beta_1^2$. Thus, the distortion constraint implies $(1-\tilde{p})\beta_0^2 + \tilde{p}\beta_1^2 \leq D$. 
\begin{Theorem}
	\label{thm:gaussianscheme1}
	For a PDD privatization scheme given by  $g(X,Y)=X + (1- Y)\beta_0 -Y\beta_1 $, $\beta_0,\beta_1\ge0$, the optimal parameters are given by
	\begin{align}
	\beta^*_0=\sqrt{\frac{\tilde{p}D}{1-\tilde{p}}}, \quad \beta^*_1=\sqrt{\frac{(1-\tilde{p})D}{\tilde{p}}}.
	\end{align}
	For this optimal PDD privatization scheme, the accuracy of the MAP adversary is given by \eqref{eq:gaussianscheme0} with $\alpha = { \frac{2\mu  -(\sqrt{\frac{(1-\tilde{p})D}{\tilde{p}}} + \sqrt{\frac{\tilde{p}D}{1-\tilde{p}}})}{\sigma} }.$
\end{Theorem}
The proof of Theorem \ref{thm:gaussianscheme1} is provided in Appendix \ref{gaussianscheme1proof}. When $P(Y=1)=P(Y=0)=\frac{1}{2}$, we have $\beta_0=\beta_1=\sqrt{D}$, which implies that the optimal privacy mechanism for this particular case is to shift the two Gaussian distributions closer to each other equally by $\sqrt{D}$ regardless of the variance $\sigma^2$. When $P(Y=1)=\tilde{p}>\frac{1}{2}$, the Gaussian distribution with a lower prior probability, in this case, $X|Y=0$, gets shifted $\frac{\tilde{p}}{1-\tilde{p}}$ times more than $X|Y=1$.


Next, we consider a slightly more complicated case in which $\gamma_0=\gamma_1=\gamma\ge0$. Thus, the privacy mechanism is given by $g(X,Y) = X + (1- Y)\beta_0 - Y\beta_1 + \gamma N$, where $N \sim \mathcal{N}(0,1)$. Intuitively, this mechanism privatizes the data by shifting the two Gaussian distributions (under $Y = 0$ and $Y = 1$) closer to each other and adding another Gaussian noise $N\in\mathcal{N}(0,1)$ scaled by a constant $\gamma$. In this case, the MAP probability of inferring the private variable $Y$ from $\hat{X}$ is given by \eqref{eq:gaussianscheme0}
with $\alpha = {\frac{2\mu  -(\beta_1 + \beta_0)}{\sqrt{\gamma^2+\sigma^2}}}$. Furthermore, the distortion constraint is equivalent to $(1-\tilde{p})\beta_0^2 + \tilde{p}\beta_1^2 +\gamma^2\leq D$.
\begin{Theorem}
	\label{thm:gaussianscheme2}
	For a PDD privatization scheme given by  $g(X,Y) = X + (1- Y)\beta_0 - Y\beta_1 + \gamma N$ with $\beta_0,\beta_1,\gamma\ge0$, the optimal parameters $\beta_0^*,\beta_1^*, \gamma^*$ are given by the solution to
	\begin{align}
	\label{eq:gaussianscheme2opt}
	\min \limits_{\beta_0,\beta_1,\gamma}& \quad { \frac{2\mu - \beta_0-\beta_1}{\sqrt{\gamma^2 + \sigma^2}}}\\\nonumber
	s.t. & \quad(1-\tilde{p})\beta_0^2 + \tilde{p}\beta_1^2 + \gamma^2 \leq D\\\nonumber
	& \quad\beta_0,\beta_1,\gamma\ge 0.
	\end{align}
	Using this optimal scheme, the accuracy of the MAP adversary is given by \eqref{eq:gaussianscheme0} with $\alpha = { \frac{2\mu - \beta_0^*-\beta_1^*}{\sqrt{(\gamma^*)^2 + \sigma^2}}}$.
\end{Theorem}
\begin{proof}
	Similar to the proofs of Theorem \ref{thm:gaussianscheme0} and \ref{thm:gaussianscheme1}, we can compute the derivative of ${P^{\text{(G)}}_{\text{d}}}$ \textit{w.r.t.}~$\alpha$. It is easy to verify that ${P^{\text{(G)}}_{\text{d}}}$ is monotonically increasing with $\alpha$. Therefore, the optimal mechanism is given by the solution to \eqref{eq:gaussianscheme2opt}. Substituting the optimal parameters into \eqref{eq:gaussianscheme0} yields the MAP probability of inferring the private variable $Y$ from $\hat{X}$.
\end{proof}
\textit{Remark:} Note that the objective function in \eqref{eq:gaussianscheme2opt} only depends on $\beta_0+\beta_1$ and $\gamma$. We define $\beta=\beta_0+\beta_1$. Thus, the above objective function can be written as
\begin{align}
	\label{eq:gaussianscheme2beta}
	\min\limits_{\beta, \gamma}{ \frac{2\mu - \beta}{\sqrt{\gamma^2 + \sigma^2}}}.
\end{align}
It is straightforward to verify that the determinant of the Hessian of \eqref{eq:gaussianscheme2beta} is always non-positive. Therefore, the above optimization problem is non-convex in $\beta$ and $\gamma$.

Finally, we consider the PDD Gaussian noise adding privatization scheme given by $g(X,Y)= X + (1-Y)\beta_0 - Y\beta_1 + (1-Y)\gamma_0N + Y \gamma_1N$, where $N \sim \mathcal{N}(0,1)$. This PDD mechanism is the most general one in the Gaussian noise adding setting and includes the two previous mechanisms. The objective of the privatizer is to minimize the adversary's probability of correctly inferring $Y$ from $g(X,Y)$ subject to the distortion constraint given by $\tilde{p}((\beta_1)^2+(\gamma_1)^2)+(1-\tilde{p})((\beta_0)^2+(\gamma_0)^2)\le D$. As we have discussed in the remark after Theorem \ref{thm:gaussianscheme2}, the problem becomes non-convex even for the simpler case in which $\gamma_0=\gamma_1=\gamma$. In order to obtain the optimal parameters for this case, we first show that the optimal privacy mechanism lies on the boundary of the distortion constraint.
\begin{Proposition}
	\label{prop:gaussian}
	For the privacy mechanism given by $g(X,Y)= X + (1-Y)\beta_0 - Y\beta_1 + (1-Y)\gamma_0N + Y \gamma_1N$, the optimal parameters $\beta^*_0,\beta^*_1,\gamma^*_0,\gamma^*_1$ satisfy $\tilde{p}((\beta^*_1)^2+(\gamma^*_1)^2)+(1-\tilde{p})((\beta^*_0)^2+(\gamma^*_0)^2)=D$.
\end{Proposition}

\begin{proof}
	We prove the above statement by contradiction. Assume that the optimal parameters satisfy $\tilde{p}((\beta^*_1)^2+(\gamma^*_1)^2)+(1-\tilde{p})((\beta^*_0)^2+(\gamma^*_0)^2)<D$. Let $\tilde{\beta}_1=\beta^*_1+c$, where $c>0$ is chosen so that $\tilde{p}((\tilde{\beta}_1)^2+(\gamma^*_1)^2)+(1-\tilde{p})((\beta^*_0)^2+(\gamma^*_0)^2)=D$. Since the inference accuracy is monotonically decreasing with $\beta_1$, the resultant inference accuracy can only be lower for replacing $\beta^*_1$ with $\tilde{\beta}_1$. This contradicts with the assumption that $\tilde{p}((\beta^*_1)^2+(\gamma^*_1)^2)+(1-\tilde{p})((\beta^*_0)^2+(\gamma^*_0)^2)<D$. Using the same type of analysis, we can show that any parameter that deviates from $\tilde{p}((\beta^*_1)^2+(\gamma^*_1)^2)+(1-\tilde{p})((\beta^*_0)^2+(\gamma^*_0)^2)=D$ is suboptimal.
\end{proof}
Let $e^2_0=(\beta^*_0)^2+(\gamma^*_0)^2$ and $e^2_1=(\beta^*_1)^2+(\gamma^*_1)^2$.
Since the optimal parameters of the privatizer lie on the boundary of the distortion constraint, we have $\tilde{p}e^2_1+(1-\tilde{p})e^2_0=D$. This implies $(e_0, e_1)$ lies on the boundary of an ellipse parametrized by $\tilde{p}$ and $D$.  Thus, we have $e_1=\sqrt{\frac{D}{\tilde{p}}}\frac{1-\epsilon^2}{1+\epsilon^2}$ and $e_0=2\sqrt{\frac{D}{1-\tilde{p}}}\frac{\epsilon}{1+\epsilon^2}$, where $\epsilon\in[0,1]$. Therefore, the optimal parameters satisfy
\begin{align}
(\beta^*_0)^2+(\gamma^*_0)^2=\left[2\sqrt{\frac{D}{1-\tilde{p}}}\frac{\epsilon}{1+\epsilon^2}\right]^2, \quad (\beta^*_1)^2+(\gamma^*_1)^2=\left[\sqrt{\frac{D}{\tilde{p}}}\frac{1-\epsilon^2}{1+\epsilon^2}\right]^2.
\end{align}
This implies $(\beta^*_i, \gamma^*_i), i\in\{0,1\}$ lie on the boundary of two circles parametrized by $D, \tilde{p}$ and $\epsilon$. Thus, we can write $\beta^*_0,\beta^*_1,\gamma^*_0,\gamma^*_1$ as
\begin{align}
&\beta^*_0=2\sqrt{\frac{D}{1-\tilde{p}}}\frac{\epsilon}{1+\epsilon^2}\frac{1-w_0^2}{1+w_0^2}, \quad\beta^*_1=\sqrt{\frac{D}{\tilde{p}}}\frac{1-\epsilon^2}{1+\epsilon^2}\frac{1-w_1^2}{1+w_1^2},\\\nonumber
&\gamma^*_0=4\sqrt{\frac{D}{1-\tilde{p}}}\frac{\epsilon}{1+\epsilon^2}\frac{w_0}{1+w_0^2}, \quad \gamma^*_1=2\sqrt{\frac{D}{\tilde{p}}}\frac{1-\epsilon^2}{1+\epsilon^2}\frac{w_1}{1+w_1^2},
\end{align}
where $\epsilon, w_0, w_1 \in[0,1]$. The optimal parameters  $\beta^*_0,\beta^*_1,\gamma^*_0,\gamma^*_1$ can be computed by a grid search in the cube parametrized by $\epsilon, w_0, w_1 \in[0,1]$ that minimizes the accuracy of the MAP adversary. In the following section, we will use this general PDD Gaussian noise adding privatization scheme in our data-driven simulations and compare the performance of the privacy mechanisms obtained by both theoretical and data-driven approaches.

\subsection{Data-driven Approach for Binary Gaussian Mixture Model}
To illustrate our data-driven GAP approach, we assume the privatizer only has access to the dataset $\mathcal{D}$ but does not know the joint distribution of $(X,Y)$. Finding the optimal privacy mechanism becomes a learning problem. In the training phase, we use the empirical log-loss function $L_{\text{XE}}(h(g(X,Y;\theta_{p});\theta_{a}),Y)$ provided in~\eqref{eq:lossCEbinary} for the adversary. Thus, for a fixed privatizer parameter $\theta_{p}$, the adversary learns the optimal parameter $\theta^*_a$ that maximizes $-L_{\text{XE}}(h(g(X,Y;\theta_{p});\theta_{a}),Y)$. On the other hand, the optimal parameter for the privacy mechanism is obtained by solving \eqref{eq:learnedprivatizer}. After convergence, we use the learned data-driven GAP mechanism to compute the accuracy of inferring the private variable under a strong MAP adversary. We evaluate our data-driven approach by comparing the mechanisms learned in an adversarial fashion on $\mathcal{D}$ with the game-theoretically optimal ones in which both the adversary and privatizer are assumed to have access to $P(X,Y)$.

\begin{figure}
	\centering
	\includegraphics[width=0.8\textwidth]{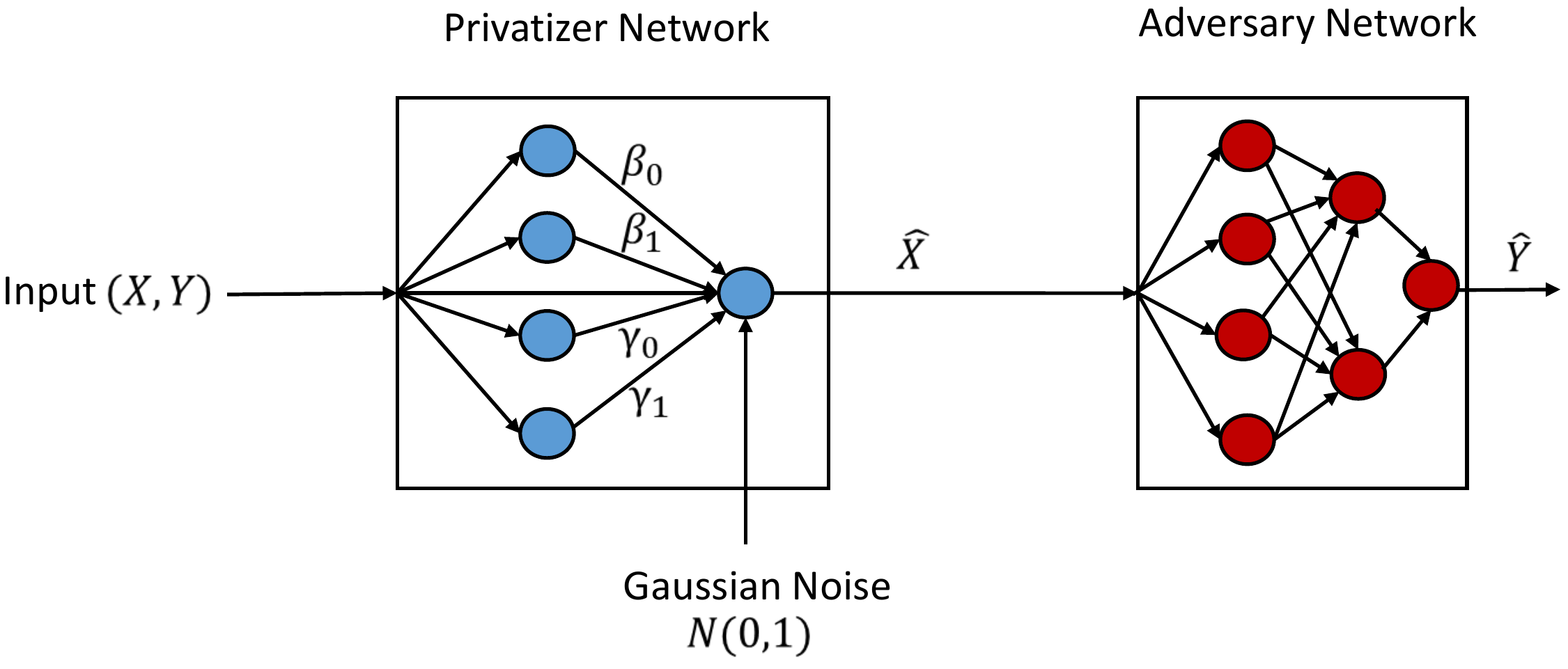}	
	\caption{Neural network structure of the privatizer and adversary for binary Gaussian mixture model}
	\label{fig:nngaussian}
\end{figure}

We consider the PDD Gaussian noise adding privacy mechanism given by $g(X,Y) =X + (1-Y)\beta_0  - Y\beta_1 + (1-Y)\gamma_0 N + Y\gamma_1N$. Similar to the binary setting, we use two neural networks to model the privatizer and the adversary. As shown in Figure \ref{fig:nngaussian}, the privatizer is modeled by a two-layer neural network with parameters $\beta_0, \beta_1, \gamma_0, \gamma_1 \in \mathbb{R}$. The adversary, whose goal is to infer $Y$ from privatized data $\hat{X}$, is modeled by a three-layer neural network classifier with leaky ReLU activations. The random noise is drawn from a standard Gaussian distribution $N \sim \mathcal{N}(0, 1)$.

In order to enforce the distortion constraint, we use the augmented Lagrangian method to penalize the learning objective when the constraint is not satisfied. In the binary Gaussian mixture model setting,  the augmented Lagrangian method uses two parameters, namely $\lambda_t$ and $\rho_t$ to approximate the constrained optimization problem by a series of unconstrained problems. Intuitively, a large value of $\rho_t$ enforces the distortion constraint to be binding, whereas $\lambda_t$ is an estimate of the Lagrangian multiplier. To obtain the optimal solution of the constrained optimization problem, we solve a series of unconstrained problems given by \eqref{eq:augmentedlagrange}.

\begin{table}[!hbpt]
	\centering
	\caption{Synthetic datasets}
	\label{tb:datasets}
	\begin{threeparttable}
		\begin{tabular}{cccc}
			\toprule
			Dataset & $P(Y=1)$  & $X|Y=0$ & $X|Y=1$  \\
			\midrule
			1&0.5& $\mathcal{N}(-3,1)$ & $\mathcal{N}(3,1)$  \\
			2&0.5& $\mathcal{N}(-3,4)$ & $\mathcal{N}(3,1)$\\
			3&0.75& $\mathcal{N}(-3,1)$ & $\mathcal{N}(3,1)$  \\
			4 &0.75 & $\mathcal{N}(-3,4)$ & $\mathcal{N}(3,1)$  \\
			\bottomrule
		\end{tabular}
	\end{threeparttable}
\end{table}
\subsection{Illustration of Results}
We use synthetic datasets to evaluate our proposed GAP framework. We consider four synthetic datasets shown in Table \ref{tb:datasets}. Each synthetic dataset used in this experiment contains $20,000$ training samples and $2,000$ test samples. We use Tensorflow to train both the privatizer and the adversary using Adam optimizer with a learning rate of $0.01$ and a minibatch size of $200$.

\begin{figure}[!hbpt]
	\centering
	\begin{subfigure}[t]{0.45\textwidth}
		\includegraphics[width=0.98\columnwidth]{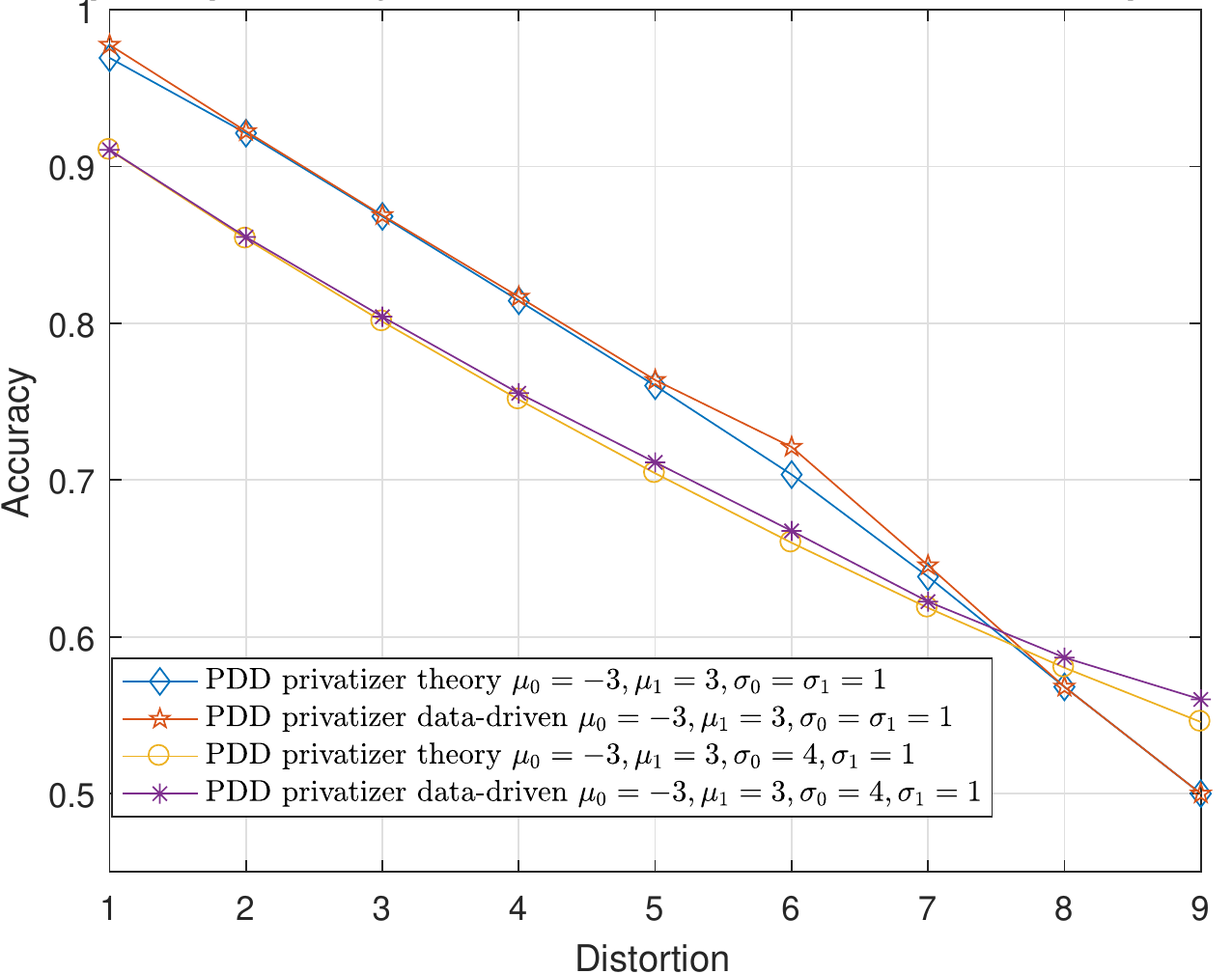}
		\caption{Performance of PDD mechanisms against MAP adversary for $\tilde{p}=0.5$}
		\label{fig:tl05gaussian}
	\end{subfigure} \qquad
	\begin{subfigure}[t]{0.45\textwidth}
		\includegraphics[width=0.98\columnwidth]{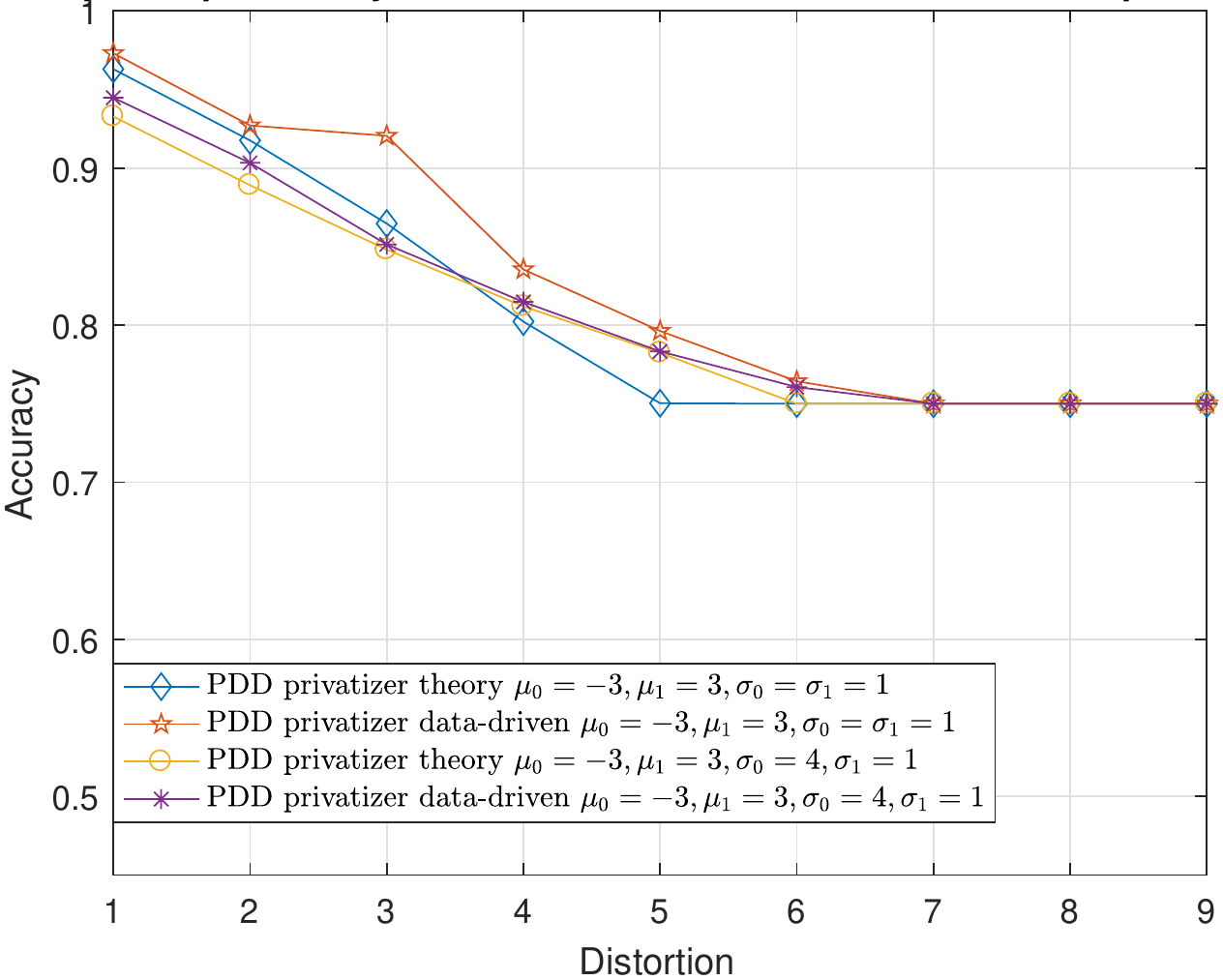}
		\caption{Performance of PDD mechanisms against MAP adversary for $\tilde{p}=0.75$}
		\label{fig:tl075gaussian}
	\end{subfigure}\\
	\caption{Privacy-distortion tradeoff for binary Gaussian mixture model}%
	\label{fig:gaussianperformancetl}
\end{figure}

Figure \ref{fig:tl05gaussian} and \ref{fig:tl075gaussian} illustrate the performance of the optimal PDD Gaussian noise adding mechanisms against the strong theoretical MAP adversary when $P(Y=1)=0.5$ and $P(Y=1)=0.75$, respectively. It can be seen that the optimal mechanisms obtained by both theoretical and data-driven approaches reduce the {inference accuracy} of the MAP adversary as the distortion increases. Similar to the binary data model, we observe that the accuracy of the adversary saturates when the distortion crosses some threshold. Moreover, it is worth pointing out that for the binary Gaussian mixture setting, we also observe that the privacy mechanism obtained through the data-driven approach performs very well when pitted against the MAP adversary (maximum accuracy difference around $6 \%$ compared with theoretical approach). In other words, for the binary Gaussian mixture model, the data-driven approach for GAP can generate privacy mechanisms that are comparable, in terms of performance, to the theoretical approach, which assumes the privatizer has access to the underlying distribution of the data.

Figures \ref{fig:gaussianstart} to \ref{fig:gaussianend} show the privatization schemes for different datasets. The intuition of this Gaussian noise adding mechanism is to shift distributions of $X|Y=0$ and $X|Y=1$ closer and scale the variances to preserve privacy. When $P(Y=0)=P(Y=1)$ and $\sigma_0=\sigma_1$, the privatizer shifts and scales the two distributions almost equally. Furthermore, the resultant $\hat{X}|Y=0$ and $\hat{X}|Y=1$ have very similar distributions. We also observe that if $P(Y=0)\neq P(Y=1)$, the public variable whose corresponding private variable has a lower prior probability gets shifted more. It is also worth mentioning that when $\sigma_0\neq \sigma_1$, the public variable with a lower variance gets scaled more.

The optimal privacy mechanisms obtained via the data-driven approach under different datasets are presented in Tables \ref{privacyPaper:table:gaussian:sim:t1} to \ref{privacyPaper:table:gaussian:sim:t4}. In each table, $D$ is the maximum allowable distortion. $\beta_0$, $\beta_1$, $\gamma_0$, and $\gamma_1$ are the parameters of the privatizer neural network. These learned parameters dictate the statistical model of the privatizer, which is used to sanitize the dataset. We use $acc$ to denote the {inference} accuracy of the adversary using a test dataset and $xent$ to denote the converged {cross-entropy} of the adversary. The column titled $distance$ represents the average distortion $\mathbb{E}_{\mathcal{D}}\lVert X-\hat{X}\rVert^2$ that results from sanitizing the test dataset via the learned privatization scheme. $P_{\text{detect}}$ is the MAP adversary's inference accuracy under the learned privatization scheme, assuming that the adversary: (a) has access to the joint distribution of $(X,Y)$, (b) has knowledge of the learned privatization scheme, and (c) can compute the MAP rule. $P_{\text{detect-theory}}$ is the ``lowest'' inference accuracy we get if the privatizer had access to the joint distribution of $(X,Y)$, and used this information to compute the parameters of the privatization scheme based on the approach provided at the end of Section \ref{sec:pddgaussian}.

\begin{figure}[!hbpt]
	\centering
	\includegraphics[width=0.5\textwidth]{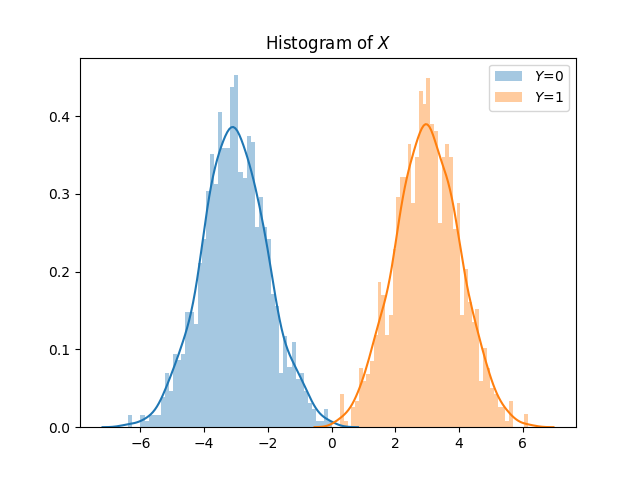}
	\caption{Raw test samples, equal variance}
	\label{fig:gaussianstart}
\end{figure}
\begin{figure}[!hbpt]
	\centering
	\begin{subfigure}[t]{0.32\textwidth}
		\includegraphics[width=0.98\columnwidth]{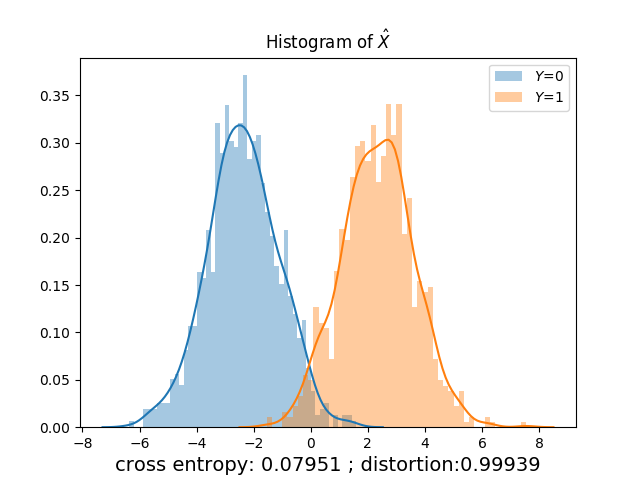}
		\caption{$D=1$}
	\end{subfigure}
	\begin{subfigure}[t]{0.32\textwidth}
		\includegraphics[width=0.98\columnwidth]{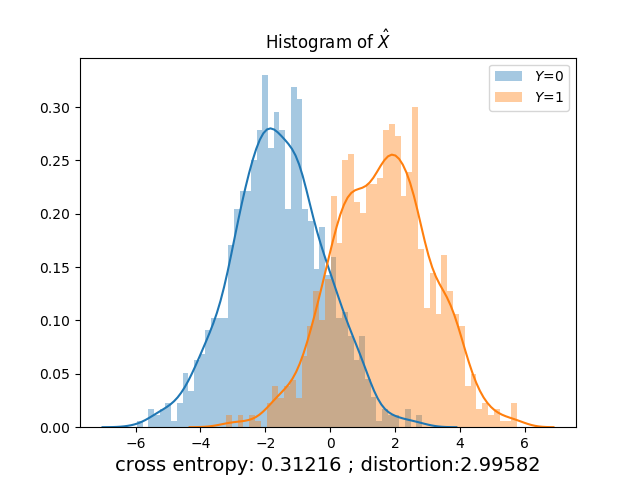}
		\caption{$D=3$}
	\end{subfigure}
	\begin{subfigure}[t]{0.32\textwidth}
		\includegraphics[width=0.98\columnwidth]{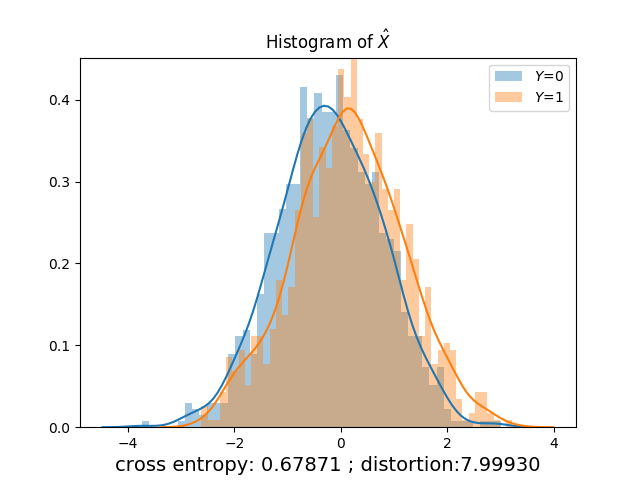}
		\caption{$D=8$}
	\end{subfigure}
	\caption{Prior $P(Y=1) = 0.5$, $X|Y=1 \sim N(3, 1)$, $X|Y=0 \sim N(-3, 1)$}
\end{figure}

\begin{figure}[!hbpt]
	\centering
	\begin{subfigure}[t]{0.32\textwidth}
		\includegraphics[width=0.98\columnwidth]{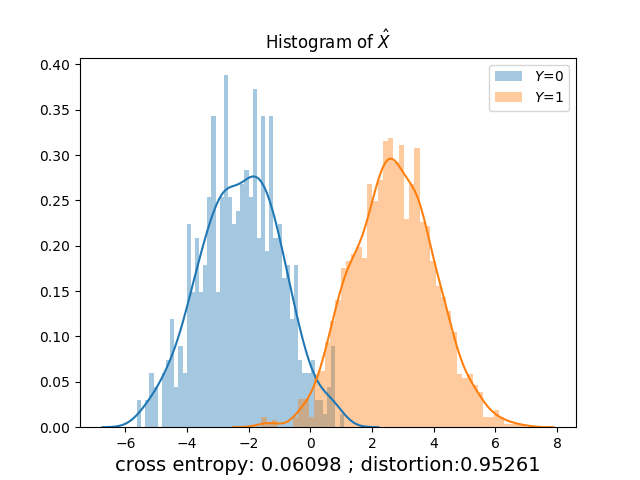}
		\caption{$D=1$}
	\end{subfigure}
	\begin{subfigure}[t]{0.32\textwidth}
		\includegraphics[width=0.98\columnwidth]{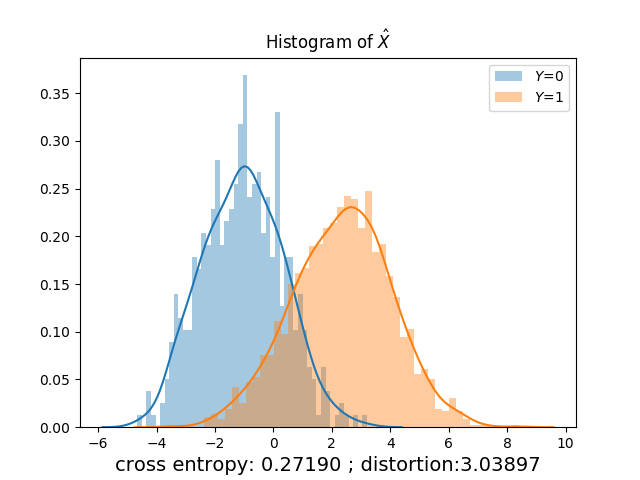}
		\caption{$D=3$}
	\end{subfigure}
	\begin{subfigure}[t]{0.32\textwidth}
		\includegraphics[width=0.98\columnwidth]{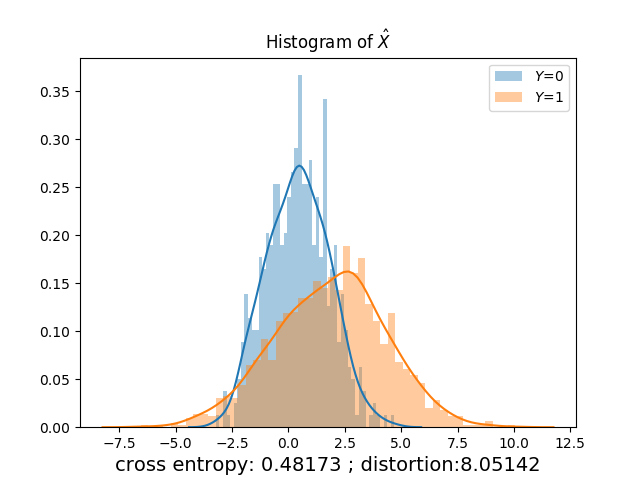}
		\caption{$D=8$}
	\end{subfigure}
	\caption{Prior $P(Y=1) = 0.75$, $X|Y=1 \sim N(3, 1)$, $X|Y=0 \sim N(-3, 1)$}
\end{figure}
%
%
\begin{figure}[!hbpt]
	\centering
	\includegraphics[width=0.5\textwidth]{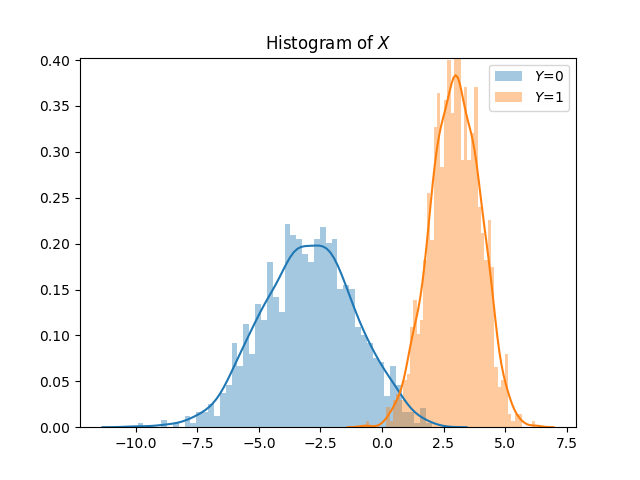}
	\caption{Raw test samples, unequal variance}
\end{figure}
%
\begin{figure}[!hbpt]
	\centering
	\begin{subfigure}[t]{0.32\textwidth}
		\includegraphics[width=0.98\columnwidth]{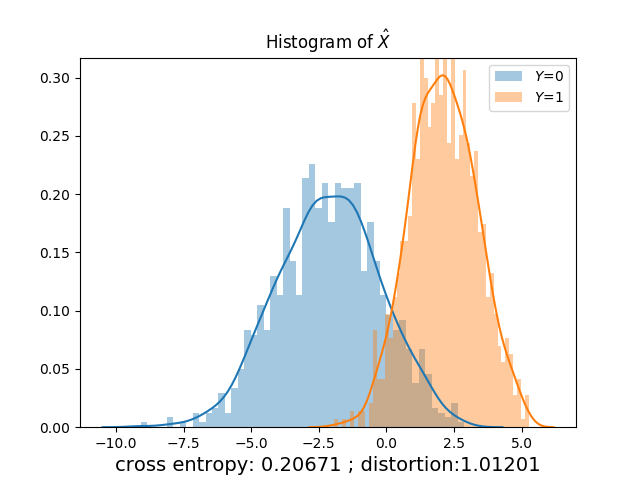}
		\caption{$D=1$}
	\end{subfigure}
	\begin{subfigure}[t]{0.32\textwidth}
		\includegraphics[width=0.98\columnwidth]{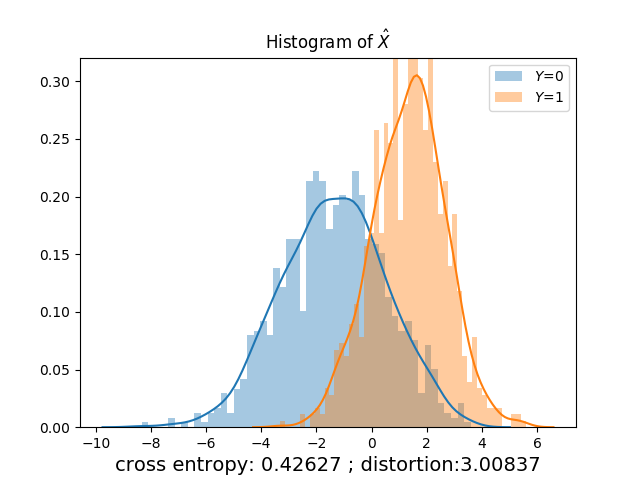}
		\caption{$D=3$}
	\end{subfigure}
	\begin{subfigure}[t]{0.32\textwidth}
		\includegraphics[width=0.98\columnwidth]{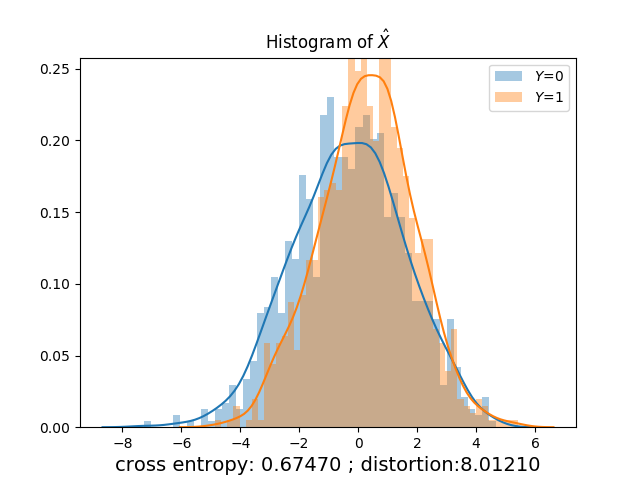}
		\caption{$D=8$}
	\end{subfigure}
	\caption{Prior $P(Y=1) = 0.5$, $X|Y=1 \sim N(3, 1)$, $X|Y=0 \sim N(-3, 4)$}
\end{figure}
%
\begin{figure}[!hbpt]
	\centering
	\begin{subfigure}[t]{0.32\textwidth}
		\includegraphics[width=0.98\columnwidth]{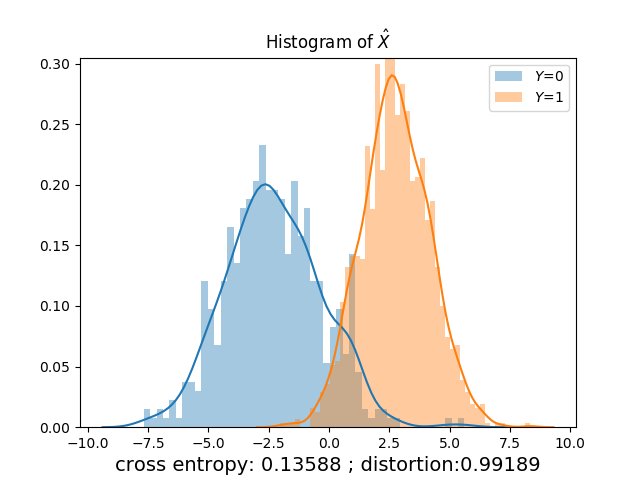}
		\caption{$D=1$}
	\end{subfigure}
	\begin{subfigure}[t]{0.32\textwidth}
		\includegraphics[width=0.98\columnwidth]{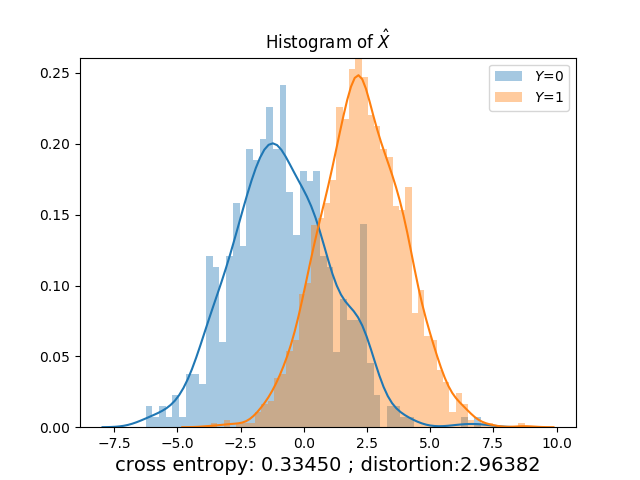}
		\caption{$D=3$}
	\end{subfigure}
	\begin{subfigure}[t]{0.32\textwidth}
		\includegraphics[width=0.98\columnwidth]{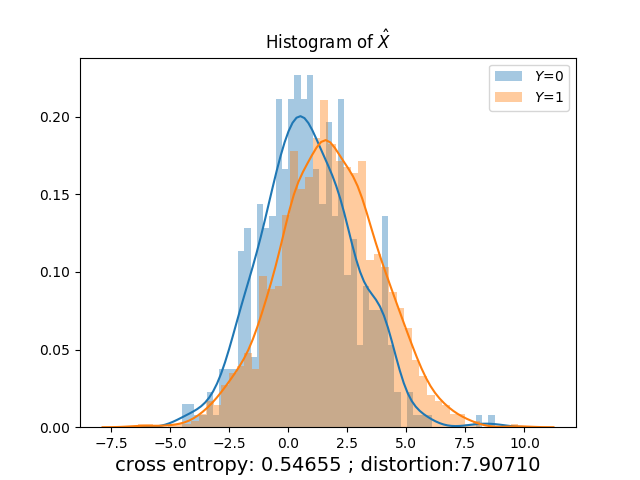}
		\caption{$D=8$}
	\end{subfigure}
	\caption{Prior $P(Y=1) = 0.75$, $X|Y=1 \sim N(3, 1)$, $X|Y=0 \sim N(-3, 4)$}
	\label{fig:gaussianend}
\end{figure}

\begin{table}[!hbpt]
	\centering
	\caption{Prior $P(Y=1) = 0.5$, $X|Y=1 \sim N(3, 1)$, $X|Y=0 \sim N(-3, 1)$}
	\label{privacyPaper:table:gaussian:sim:t1}
	\begin{threeparttable}
		\begin{tabular}{lrrrrrrrrr}
			\toprule
			$D$ & $\beta_0$ & $\beta_1$ & $\gamma_0$ & $\gamma_1$ & $acc$ & $xent$ & $distance$ & $P_{detect}$ & $P_{detect -theory}$ \\
			\midrule
			1&0.5214& 0.5214 & 0.7797 &0.7797 & 0.9742 & 0.0715 & 0.9776 & 0.9747 &0.9693 \\
			2&0.9861&0.9861&1.0028&1.0029&0.9169&0.1974&1.9909&0.9225 &0.9213 \\
			3&1.3819&1.3819&1.0405&1.0403&0.8633&0.3130&3.0013&0.8689 &0.8682 \\
			4&1.5713&1.5713&1.2249&1.2249&0.8123&0.4066&4.0136&0.8169 &0.8144 \\
			5&1.8199&1.8199&1.3026&1.3024&0.7545&0.4970&4.9894&0.7638 &0.7602 \\
			6&1.9743&1.9745&1.436&1.4359&0.7122&0.5564&5.9698&0.7211&0.7035 \\
			7&2.5332&2.5332&0.7499&0.7500&0.6391&0.6326&7.0149&0.6456 &0.6384\\
			8&2.8284&2.8284&0.0044&0.0028&0.5727&0.6787&7.9857&0.5681 &0.5681 \\
			9&2.9999&3.0000&0.0003&0.0004&0.4960&0.6938&8.9983&0.5000&0.5000 \\
			\bottomrule
		\end{tabular}
	\end{threeparttable}
\end{table}

\begin{table}[!hbpt]
	\centering
	\caption{Prior $P(Y=1) = 0.75$, $X|Y=1 \sim N(3, 1)$, $X|Y=0 \sim N(-3, 1)$}
	\label{privacyPaper:table:gaussian:sim:t2}
	\begin{threeparttable}
		\begin{tabular}{lrrrrrrrrr}
			\toprule
			$D$ & $\beta_0$ & $\beta_1$ & $\gamma_0$ & $\gamma_1$ & $acc$ & $xent$ & $distance$ & $P_{detect}$ & $P_{detect -theory}$ \\
			\midrule
			1&0.8094&0.2698&0.844&0.8963&0.9784&0.0591&0.9533&0.9731 & 0.9630 \\
			2&1.4998&0.5000&0.9676&1.1612&0.9314&0.1635&1.9098&0.9271 &0.9176\\
			3&0.9808&0.3269&1.3630&1.5762&0.911&0.2054&2.9833&0.9205&0.8647\\
			4&2.2611&0.7536&1.1327&1.6225&0.8359&0.3519&4.0559&0.8355&0.8023\\
			5&2.5102&0.8368&1.0724&1.8666&0.792&0.401&5.0445&0.7963&0.7503 \\
			6&2.8238&0.9412&1.2894&1.9752&0.7627&0.4559&6.0843&0.7643 &0.7500 \\
			7&3.2148&1.0718&0.6938&2.1403&0.7500&0.4468&7.0131&0.7500&0.7500\\
			8&3.3955&1.1320&1.0256&2.2789&0.7500&0.4799&8.0484&0.7500&0.7500\\
			9&4.1639&1.3878&0.0367&2.0714&0.7500&0.4745&8.9343&0.7500&0.7500 \\
			\bottomrule
		\end{tabular}
	\end{threeparttable}
\end{table}


\begin{table}[!hbpt]
	\centering
	\caption{Prior $P(Y=1) = 0.5$, $X|Y=1 \sim N(3, 1)$, $X|Y=0 \sim N(-3, 4)$}
	\label{privacyPaper:table:gaussian:sim:t3}
	\begin{threeparttable}
		\begin{tabular}{lrrrrrrrrr}
			\toprule
			$D$ & $\beta_0$ & $\beta_1$ & $\gamma_0$ & $\gamma_1$ & $acc$ & $xent$ & $distance$ & $P_{detect}$ & $P_{detect -theory}$ \\
			\midrule
			1&0.8660&0.8660&0.0079&0.7074&0.9122&0.2103&1.0078&0.9107&0.9105 \\
			2&1.2781&1.2781&0.0171&0.8560&0.8595&0.3239&2.0181&0.8550&0.8539 \\
			3&1.5146&1.5146&0.0278&1.1352&0.8084&0.4211&3.0264&0.8042&0.8011 \\
			4&1.7587&1.7587&0.0330&1.2857&0.7557&0.4970&4.0274&0.7554&0.7513 \\
			5&2.0923&2.0923&0.0142&1.0028&0.7057&0.5589&5.0082&0.7113&0.7043 \\
			6&2.3079&2.2572&0.0211&1.1185&0.6650&0.5999&6.0377&0.6676&0.6600 \\
			7&2.5351&2.5351&0.0567&1.0715&0.6100&0.6509&7.0125&0.6225&0.6185\\
			8&2.7056&2.7056&0.0358&1.1665&0.5770&0.6738&8.0088&0.5868&0.5803\\
			9&2.8682&2.8682&0.0564&1.2435&0.5445&0.6844&9.0427&0.5601&0.5457 \\
			\bottomrule
		\end{tabular}
	\end{threeparttable}
\end{table}

\begin{table} [H]
	\centering
	\caption{Prior $P(Y=1) = 0.75$, $X|Y=1 \sim N(3, 1)$, $X|Y=0 \sim N(-3, 4)$}
	\label{privacyPaper:table:gaussian:sim:t4}
	\begin{threeparttable}
		\begin{tabular}{lrrrrrrrrr}
			\toprule
			$D$ & $\beta_0$ & $\beta_1$ & $\gamma_0$ & $\gamma_1$ & $acc$ & $xent$ & $distance$ & $P_{detect}$ & $P_{detect -theory}$ \\
			\midrule
			1&0.8214&0.2739&0.0401&1.0167&0.9514&0.1357&0.9909&0.9448&0.9328 \\
			2&1.4164&0.4722&0.0583&1.2959&0.9026&0.2402&2.0257&0.9033&0.8891\\
			3&2.2354&0.7450&0.0246&1.3335&0.8665&0.3354&2.9617&0.8514&0.8481 \\
			4&2.6076&0.8693&0.0346&1.5199&0.8269&0.4034&3.9522&0.8148&0.8120 \\
			5&2.9919&0.9977&0.0143&1.6399&0.7885&0.4625&5.0034&0.7833&0.7824 \\
			6&3.3079&1.1027&0.0094&1.7707&0.7616&0.5013&6.0022&0.7606&0.7500 \\
			7&3.1458&1.0488&0.0565&2.1606&0.7496&0.4974&7.0091&0.7500&0.7500 \\
			8&3.9707&1.3237&0.0142&1.9129&0.7500&0.5470&7.9049&0.7500&0.7500 \\
			9&4.0835&1.3613&0.0625&2.1364&0.7500&0.5489&8.8932&0.7500&0.7500\\
			\bottomrule
		\end{tabular}
	\end{threeparttable}
\end{table}

\section{Concluding Remarks}
\label{sec:conclusion}
We have presented a unified framework for context-aware privacy called generative adversarial privacy (GAP). GAP allows the data holder to learn the privatization mechanism directly from the dataset (to be published) without requiring access to the dataset statistics. Under GAP, finding the optimal privacy mechanism is formulated as a game between two players: a privatizer and an adversary. An iterative minimax algorithm is proposed to obtain the optimal mechanism under the GAP framework.

To evaluate the performance of the proposed GAP model, we have focused on two types of datasets: (i) binary data model; and (ii) binary Gaussian mixture model. For both cases, the optimal GAP mechanisms are learned using an empirical $\log$-loss function. For each type of dataset, both private-data dependent and private-data independent mechanisms are studied. These results are cross-validated against the privacy guarantees obtained by computing the game-theoretically optimal mechanism under a strong MAP adversary. In the MAP adversary setting, we have shown that for the binary data model, the optimal GAP mechanism is obtained by solving a linear program. For the binary Gaussian mixture model, the optimal additive Gaussian noise privatization scheme is determined. Simulations with synthetic datasets for both types (i) and (ii) show that the privacy mechanisms learned via the GAP framework perform as well as the mechanisms obtained from theoretical computation.

Binary and Gaussian models are canonical models with a wide range of applications. However, moving next, we would like to consider more sophisticated dataset models that can capture real life signals (such as time series data and images). The generative models we have considered in this paper were tailored to the statistics of the datasets. In the future, we would like to experiment with the idea of using a deep generative model to automatically generate the sanitized data. Another straightforward extension to our work is to use the GAP framework to obtain data-driven mutual information privacy mechanisms. Finally, it would be interesting to investigate adversarial {loss} functions that allow us to move from weak to strong adversaries.

%

\appendix
\section{Proof of Theorem~\ref{thm:binary}}
\label{binaryproof}
\begin{proof}
If $q=\frac{1}{2}$, $X$ is independent of $Y$. The optimal solution is given by any $(s_0, s_1)$ that satisfies the distortion constraint ($\{s_0,s_1| ps_1+(1-p)s_0\ge 1-D, s_0, s_1\in[0,1]\}$) since $X$ and $Y$ are already independent. If $q\neq\frac{1}{2}$, since each maximum in \eqref{eq:optprivatizer} can only be one of the two values (i.e., the {inference accuracy} of guessing $\hat{Y}=0$ or $\hat{Y}=1$), the objective function of the privatizer is determined by the relationship between $P(Y=1, \hat{X}=i)$ and $P(Y=0, \hat{X}=i), i\in\{0,1\}$. Therefore, the optimization problem in \eqref{eq:optprivatizer} can be decomposed into the following four subproblems:
	\vspace {6pt}
	
\noindent	\textit{\textbf{Subproblem 1}}: $P(Y=1,\hat{X}=0)\ge P(Y=0,\hat{X}=0)$ and $P(Y=1,\hat{X}=1)\le P(Y=0,\hat{X}=1)$, which~implies $p(1-2q)(1-s_1)-(1-p)(1-2q)s_0\ge 0$ and $ (1-p)(1-2q)(1-s_0)-p(1-2q)s_1\ge 0$. As a result, the objective of the privatizer is given by $P(Y=1,\hat{X}=0)+P(Y=0,\hat{X}=1)$.  Thus, the~optimization problem in \eqref{eq:optprivatizer} can be written as
	\begin{equation}
\begin{array}{llllll}
	\label{eq:toycase1originalobjectives0s1}
	\min\limits_{s_0,s_1}\quad& (2q-1)[ps_1+(1-p)s_0]+1-q \\
	s.t.\quad & 0\le s_0\le 1 \\
	& 0\le s_1 \le 1\\
	& p(1-2q)s_1+(1-p)(1-2q)s_0\le p(1-2q)\\
	& p(1-2q)s_1+(1-p)(1-2q)s_0\le (1-p)(1-2q)\\
	& -ps_1-(1-p)s_0\le D-1.
	\end{array}
\end{equation}
	
	\begin{itemize}
		\item If $1-2q>0$, i.e., $q<\frac{1}{2}$, we have $ps_1+(1-p)s_0\le p$ and $ps_1+(1-p)s_0\le 1-p$. The~privatizer must maximize $ps_1+(1-p)s_0$ to reduce the adversary's probability of correctly inferring the private variable. Thus, if $1-D\le\min\{p,1-p\}$, the optimal value is given by $(2q-1)\min\{p,1-p\}+1-q$; the corresponding optimal solution is given by $\{s_0,s_1|ps_1+(1-p)s_0=\min\{p,1-p\}, 0\le s_0, s_1\le 1\}$. Otherwise, the problem is infeasible. 	
		\item {If $1-2q<0$, i.e., $q>\frac{1}{2}$, we have $ps_1+(1-p)s_0\ge p$ and $ps_1+(1-p)s_0\ge 1-p$. In this case, the privatizer has to minimize $ps_1+(1-p)s_0$. Thus, if $1-D\ge\max\{p,1-p\}$, the optimal value is given by $(2q-1)(1-D)+1-q$; the corresponding optimal solution is $\{s_0,s_1|  ps_1+(1-p)s_0=1-D, 0\le s_0, s_1\le 1\}$. Otherwise, the optimal value is $(2q-1)\max\{p,1-p\}+1-q$ and the corresponding optimal solution is given by $\{s_0,s_1| ps_1+(1-p)s_0=\max\{p,1-p\}, 0\le s_0, s_1\le 1 \}$.}
	\end{itemize}
\vspace {6pt}
	
\noindent	\textit{\textbf{Subproblem 2}}: $P(Y=1,\hat{X}=0)\le P(Y=0,\hat{X}=0)$ and $P(Y=1,\hat{X}=1)\ge P(Y=0,\hat{X}=1)$, which implies $p(1-2q)(1-s_1)-(1-p)(1-2q)s_0\le 0$ and $(1-p)(1-2q)(1-s_0)-p(1-2q)s_1\le 0$. Thus, the objective of the privatizer is given by $P(Y=0,\hat{X}=0)+P(Y=1,\hat{X}=1)$.  Therefore, the optimization problem in \eqref{eq:optprivatizer} can be written as
	\begin{equation}
\begin{array}{llllll}
	\label{eq:toycase2originalobjectives0s1}
	\min\limits_{s_0,s_1}\quad& (1-2q)[ps_1+(1-p)s_0]+q\\
	s.t.\quad & 0\le s_0\le 1\\
	& 0\le s_1 \le 1\\
	& -p(1-2q)s_1-(1-p)(1-2q)s_0\le -p(1-2q)\\
	& -p(1-2q)s_1-(1-p)(1-2q)s_0\le -(1-p)(1-2q)\\
	& -ps_1-(1-p)s_0\le D-1.
	\end{array}
\end{equation}
	\begin{itemize}
		\item If $1-2q>0$, i.e., $q<\frac{1}{2}$, we have $ps_1+(1-p)s_0\ge p$ and $ps_1+(1-p)s_0\ge 1-p$. The privatizer needs to minimize $ps_1+(1-p)s_0$ to reduce the adversary's probability of correctly inferring the private variable. Thus, if $1-D\ge\max\{p,1-p\}$, the optimal value is given by $(1-2q)(1-D)+q$; the corresponding optimal solution is $\{s_0,s_1|  ps_1+(1-p)s_0=1-D, 0\le s_0,s_1\le 1\}$. Otherwise, the optimal value is $(1-2q)\max\{p,1-p\}+q$ and the corresponding optimal solution is given by $\{s_0,s_1| ps_1+(1-p)s_0=\max\{p,1-p\}, 0\le s_0, s_1\le 1 \}$.
		\item If $1-2q<0$, i.e., $q>\frac{1}{2}$, we have $ps_1+(1-p)s_0\le p$ and $ps_1+(1-p)s_0\le 1-p$. In this case, the privatizer needs to maximize $ps_1+(1-p)s_0$. Thus, if $1-D\le\min\{p,1-p\}$, the optimal value is given by $(1-2q)\min\{p,1-p\}+q$; the corresponding optimal solution is given by $\{s_0,s_1|ps_1+(1-p)s_0=\min\{p,1-p\}, 0\le s_0,s_1\le 1\}$. Otherwise, the problem is infeasible. 	
	\end{itemize}	
		\vspace {6pt}
	
\noindent		\textit{\textbf{Subproblem 3}}: $P(Y=1,\hat{X}=0)\ge P(Y=0,\hat{X}=0)$ and $P(Y=1,\hat{X}=1)\ge P(Y=0,\hat{X}=1)$, we have $ p(1-2q)(1-s_1)-(1-p)(1-2q)s_0\ge 0$ and $(1-p)(1-2q)(1-s_0)-p(1-2q)s_1\le 0$. Under this scenario, the objective function  in \eqref{eq:optprivatizer} is given by $P(Y=1,\hat{X}=0)+P(Y=1,\hat{X}=1)$. Thus,~the~privatizer solves
	\begin{equation}
\begin{array}{llllll}
	\label{eq:toycase3originalobjectives0s1}
	\min\limits_{s_0, s_1}\quad & p(1-q)+(1-p)q\\
	s.t.\quad & 0\le s_0\le 1\\
	& 0\le s_1 \le 1\\
	& p(1-2q)s_1+(1-p)(1-2q)s_0\le p(1-2q)\\
	& -p(1-2q)s_1- (1-p)(1-2q)s_0\le -(1-p)(1-2q)\\
	& -ps_1-(1-p)s_0\le D-1.
	\end{array}
\end{equation}
	\begin{itemize}
		\item If $1-2q>0$, i.e., $q<\frac{1}{2}$, the problem becomes infeasible for $p<\frac{1}{2}$. For $p\ge\frac{1}{2}$, if~$1-D>\max\{p,1-p\}$, the problem is also infeasible; if $\min\{p,1-p\}\le1-D\le \max\{p,1-p\}$, the~optimal value is given by $p(1-q)+(1-p)q$ and the corresponding optimal solution is $\{s_0,s_1| 1-D \le ps_1+(1-p)s_0\le\max\{p,1-p\}, 0\le s_0, s_1\le 1\}$; otherwise, the optimal value is $p(1-q)+(1-p)q$ and the corresponding optimal solution is given by $\{s_0,s_1| \min\{p,1-p\}\le ps_1+(1-p)s_0\le\max\{p,1-p\} , 0\le s_0,s_1\le 1 \}$.
		\item If $1-2q<0$, i.e., $q>\frac{1}{2}$, the problem is infeasible for $p>\frac{1}{2}$. For $p\le\frac{1}{2}$, if $1-D>\max\{p,1-p\}$, the problem is also infeasible; if $\min\{p,1-p\}\le1-D\le \max\{p,1-p\}$, the optimal value is given by $p(1-q)+(1-p)q$ and the corresponding optimal solution is $\{s_0,s_1|  1-D\le ps_1+(1-p)s_0\le\max\{p,1-p\}, 0\le s_0,s_1\le 1\}$; otherwise, the optimal value is $p(1-q)+(1-p)q$ and the corresponding optimal solution is given by $\{s_0,s_1| \min\{p,1-p\}\le ps_1+(1-p)s_0\le\max\{p,1-p\} , 0\le s_0, s_1\le 1\}$. 	
	\end{itemize}	
			\vspace {6pt}
	
\noindent	 	\textit{\textbf{Subproblem 4}}: $P(Y=1,\hat{X}=0)\le P(Y=0,\hat{X}=0)$ and $P(Y=1,\hat{X}=1)\le P(Y=0,\hat{X}=1)$, which implies $ p(1-2q)(1-s_1)-(1-p)(1-2q)s_0\le 0$ and $(1-p)(1-2q)(1-s_0)-p(1-2q)s_1\ge 0$. Thus, the optimization problem in \eqref{eq:optprivatizer} is given by
		\begin{equation}
\begin{array}{llllll}
	\label{eq:toycase4originalobjectives0s1}
	\min\limits_{s_0, s_1}\quad& pq+(1-p)(1-q)\\
	s.t.\quad & 0\le s_0\le 1\\
	& 0\le s_1 \le 1\\
	& -p(1-2q)s_1-(1-p)(1-2q)s_0\le -p(1-2q)\\
	& p(1-2q)s_1+(1-p)(1-2q)s_0\le (1-p)(1-2q) \\
	& -ps_1-(1-p)s_0\le D-1.
	\end{array}
\end{equation}
	\begin{itemize}
		\item If $1-2q>0$, i.e., $q<\frac{1}{2}$, the problem becomes infeasible for $p>\frac{1}{2}$. For $p\le\frac{1}{2}$, if~$1-D>\max\{p,1-p\}$, the problem is also infeasible; if $\min\{p,1-p\}\le1-D\le \max\{p,1-p\}$, the~optimal value is given by $pq+(1-p)(1-q)$ and the corresponding optimal solution is $\{s_0,s_1| 1-D\le ps_1+(1-p)s_0\le\max\{p,1-p\}, 0\le s_0,s_1\le 1\}$; otherwise, the optimal value is $pq+(1-p)(1-q)$ and the corresponding optimal solution is given by $\{s_0,s_1| \min\{p,1-p\}\le ps_1+(1-p)s_0\le\max\{p,1-p\} , 0\le s_0,s_1\le 1 \}$.
		\item If $1-2q<0$, i.e., $q>\frac{1}{2}$, the problem becomes infeasible for $p<\frac{1}{2}$. For $p\ge\frac{1}{2}$, if~$1-D>\max\{p,1-p\}$, the problem is also infeasible; if $\min\{p,1-p\}\le1-D\le \max\{p,1-p\}$, the~optimal value is given by $pq+(1-p)(1-q)$ and the corresponding optimal solution is $\{s_0,s_1| 1-D\le ps_1+(1-p)s_0\le\max\{p,1-p\}, 0\le s_0,s_1\le 1\}$; otherwise, the optimal value is $pq+(1-p)(1-q)$ and the corresponding optimal solution is given by $\{s_0,s_1| \min\{p,1-p\}\le ps_1+(1-p)s_0\le\max\{p,1-p\} , 0\le s_0,s_1\le 1 \}$. 	
	\end{itemize}	 	 	
	
	Summarizing the analysis above yields Theorem  \ref{thm:binary}.
\end{proof}

\section{Proof of Theorem~\ref{thm:gaussianscheme0}}
\label{gaussianscheme0proof}
\begin{proof}
	Let us consider $\hat{X} = X + \beta+\gamma N$, where $\beta\in\mathbb{R}$ and $\gamma
	\ge 0$. Given the MAP adversary's {optimal inference accuracy} in \eqref{eq:gaussianscheme0}, the objective of the privatizer is to
	\begin{align}
		\label{eq:gaussianscheme0opt}
		\min\limits_{\beta,\gamma} & \quad {P^{\text{(G)}}_{\text{d}}}\\\nonumber
		s.t. & \quad\beta^2 + \gamma^2 \leq D\\\nonumber
		& \quad\gamma\ge 0.
	\end{align}
	Define $\frac{1-\tilde{p}}{\tilde{p}}=\eta$.
	The gradient of ${P^{\text{(G)}}_{\text{d}}}$ \textit{w.r.t.} $\alpha$ is given by
	\begin{align}
		\frac{\partial {P^{\text{(G)}}_{\text{d}}}}{\partial \alpha}=& \tilde{p}\left(-\frac{1}{\sqrt{2\pi}}e^{-\frac{\left(-  {\frac{\alpha}{2} }  + {\frac{1}{\alpha}} \ln\eta \right) ^2}{2}}\right)\left(-  { \frac{1}{2} }  - {\frac{1}{{\alpha}^2} } \ln\eta \right)
		\\\nonumber &
		+(1-\tilde{p})\left(-\frac{1}{\sqrt{2\pi}}e^{-\frac{\left(-  { \frac{\alpha}{2} }  - { \frac{1}{\alpha} } \ln\eta \right) ^2}{2}}\right)\left(-  { \frac{1}{2} }  + { \frac{1}{{\alpha}^2} } \ln\eta \right)\\\label{eq:gaussianscheme0der}
		=& \frac{1}{2\sqrt{2\pi}}\left(\tilde{p}e^{-\frac{\left(-  { \frac{\alpha}{2} }  + { \frac{1}{\alpha} } \ln\eta \right) ^2}{2}}+(1-\tilde{p})e^{-\frac{\left(-  { \frac{\alpha}{2} }  - { \frac{1}{\alpha} } \ln\eta \right) ^2}{2}}\right)\\\nonumber & +\frac{\ln\eta}{\alpha^2\sqrt{2\pi}}\left(\tilde{p}e^{-\frac{\left(-  { \frac{\alpha}{2} }  + { \frac{1}{\alpha}  } \ln\eta \right) ^2}{2}}-(1-\tilde{p})e^{-\frac{\left(-  { \frac{\alpha}{2} }  - { \frac{1}{\alpha} } \ln\eta \right) ^2}{2}}\right).
	\end{align}
	Note that
	\begin{align}
		\frac{\tilde{p}e^{-\frac{\left(-  { \frac{\alpha}{2} }  + { \frac{1}{\alpha} } \ln\eta \right) ^2}{2}}}{(1-\tilde{p})e^{-\frac{\left(-  { \frac{\alpha}{2} }  - { \frac{1}{\alpha} } \ln\eta \right) ^2}{2}}}=\frac{\tilde{p}}{1-\tilde{p}}e^{\frac{\left(-  { \frac{\alpha}{2} }  - { \frac{1}{\alpha} } \ln\eta \right) ^2-\left(-  { \frac{\alpha}{2} }  + { \frac{1}{\alpha} } \ln\eta \right) ^2}{2}}=\frac{\tilde{p}}{1-\tilde{p}}e^{\frac{2\ln\eta}{2}}=\frac{\tilde{p}}{1-\tilde{p}}e^{\ln\eta}=1.
	\end{align}
	Therefore, the second term in \eqref{eq:gaussianscheme0der} is $0$. Furthermore, the first term in \eqref{eq:gaussianscheme0der} is always positive. Thus, ${P^{\text{(G)}}_{\text{d}}}$ is monotonically increasing in $\alpha$. As a result, the optimization problem in \eqref{eq:gaussianscheme0opt} is equivalent to
	\begin{align}
		\label{eq:gaussianscheme0opteq}
		\max\limits_{\beta,\gamma} & \quad\sqrt{\gamma^2 + \sigma^2}\\\nonumber
		s.t. & \quad\beta^2+\gamma^2 \leq D\\\nonumber
		& \quad\gamma\ge 0.
	\end{align}
	Therefore, the optimal solution is given by $\beta^*=0$ and $\gamma^*=\sqrt{D}$. Substituting the optimal solution back into \eqref{eq:gaussianscheme0} yields the MAP probability of correctly inferring the private variable $Y$ from $\hat{X}$.
\end{proof}

\section{Proof of Theorem~\ref{thm:gaussianscheme1}}
\label{gaussianscheme1proof}
\begin{proof}
	Let us consider $\hat{X} = X + (1- Y)\beta_0 - Y\beta_1$, where $\beta_0$ and $\beta_1$ are both non-negative. Given the MAP adversary's {optimal inference accuracy} ${P^{\text{(G)}}_{\text{d}}}$, the objective of the privatizer is to
	\begin{align}
	\label{eq:gaussianscheme1opt}
	\min\limits_{\beta_0,\beta_1} & \quad {P^{\text{(G)}}_{\text{d}}}\\\nonumber
	s.t. & \quad(1-\tilde{p})\beta_0^2 + \tilde{p}\beta_1^2 \leq D\\\nonumber
	& \quad\beta_0,\beta_1\ge 0.
	\end{align} Recall that ${P^{\text{(G)}}_{\text{d}}}$ is monotonically increasing in $\alpha = { \frac{2\mu  -(\beta_1 + \beta_0)}{\sigma} }$. As a result, the optimization problem in \eqref{eq:gaussianscheme1opt} is equivalent to
	\begin{align}
	\label{eq:gaussianscheme1opteq}
	\max\limits_{\beta_0,\beta_1} & \quad\beta_1 + \beta_0\\\nonumber
	s.t. & \quad(1-\tilde{p})\beta_0^2 + \tilde{p}\beta_1^2 \leq D\\\nonumber
	& \quad\beta_0,\beta_1\ge 0.
	\end{align}
	Note that the above optimization problem is convex. Therefore, using the KKT conditions, we obtain the optimal solution
	\begin{align}
	\beta^*_0=\sqrt{\frac{\tilde{p}D}{1-\tilde{p}}}, \quad \beta^*_1=\sqrt{\frac{(1-\tilde{p})D}{\tilde{p}}}.
	\end{align}
	Substituting the above optimal solution into ${P^{\text{(G)}}_{\text{d}}}$ yields the MAP probability of {correctly inferring} the private variable $Y$ from $\hat{X}$.
\end{proof}

\bibliographystyle{plainnat}
\bibliography{sample,references,LS_Privacy_Refs,StatLearning_Bibliography}



\end{document}